\documentclass{article} 
\usepackage{iclr2027_conference,times}

\usepackage{amsmath,amsfonts,bm}

\def\eqref#1{equation~\ref{#1}}

\def\1{\bm{1}}

\DeclareMathAlphabet{\mathsfit}{\encodingdefault}{\sfdefault}{m}{sl}
\SetMathAlphabet{\mathsfit}{bold}{\encodingdefault}{\sfdefault}{bx}{n}

\usepackage{hyperref}
\usepackage{url}
\usepackage{booktabs} 
\usepackage{amsthm}
\usepackage{graphicx}
\usepackage{subcaption}
\newtheorem{theorem}{Theorem}
\newtheorem{proposition}{Proposition}
\newtheorem{lemma}{Lemma}

\usepackage{algorithm}
\usepackage{algpseudocode}
\usepackage{multirow}
\usepackage{wrapfig}

\title{LSC-DPO: Learning-Signal-Controlled \\ Direct Preference Optimization}

\author{
	Yang Qu$^{1}$ \quad
	Yusheng Han$^{1}$ \quad
	Chengjia Feng$^{1}$ \quad
	Handan Liu$^{*}$ \\
	$^{1}$Northeastern University \\
	\texttt{\{qu.yang1, han.yus, feng.chen, h.liu\}@northeastern.edu} \\
	$^{*}$Corresponding author
}

\iclrfinalcopy 
\begin{document}

\maketitle
\lhead{}

\begin{abstract}
Direct Preference Optimization (DPO) has become a standard reward-model-free approach for aligning language models with preference data. However, as the scaled preference margin grows during training, the logistic DPO loss becomes progressively less sensitive to further changes. We study DPO from a loss-level geometric perspective and identify the sigmoid factor as a learning signal that characterizes the local sensitivity of the objective. Based on this view, we propose Learning-Signal-Controlled Direct Preference Optimization (LSC-DPO), which dynamically regulates the learning signal near a target regime. A log-space analysis establishes conditions for stable tracking of the target learning-signal regime. Experiments on AlpacaEval~2, MT-Bench, and Anthropic-HH show that LSC-DPO consistently improves over DPO and strong preference-optimization baselines. We further find that different coefficient initializations induce distinct transient learning-signal
trajectories even when their later signal levels become similar. Based on this observation, we derive a signal-budget compensation rule that adjusts the target learning signal to compensate for these transient differences. The resulting compensation substantially reduces performance variation
across coefficient initializations.\footnote{Code is available at \url{https://github.com/handan-liu-lab/lsc-dpo}.}
\end{abstract}

\section{Introduction}
\label{Intro}
Large language models (LLMs) have achieved remarkable performance across a broad range of tasks. However, their outputs are not always aligned with human preferences. Reinforcement learning from human feedback (RLHF) \citep{christiano2017deep, ouyang2022training, stiennon2020learning} has become a widely used paradigm for improving alignment. In the standard RLHF pipeline, approaches such as Proximal Policy Optimization (PPO) \citep{schulman2017proximal} optimize a policy based on a learned reward model that captures human preferences, while a KL penalty constrains the policy from deviating excessively from a reference model. 

Recent work has explored reward-model-free approaches to reduce the computational cost of RLHF. Direct Preference Optimization (DPO)~\citep{rafailov2023direct} directly optimizes a policy from offline preference data without explicitly training a reward model. In DPO, preference learning is formulated through the policy-reference log-ratio difference between chosen and rejected responses. Under the Bradley--Terry model~\citep{bradley1952rank}, the scaled log-ratio difference can be interpreted as an implicit reward difference. The coefficient $\beta$ scales this quantity and thereby affects the resulting optimization dynamics.

A central question in DPO is how to characterize and regulate its evolving optimization state during training. The DPO objective contains a logistic loss term whose sensitivity decreases as optimization progresses~\citep{rafailov2023direct, rafailov2024scaling, gorbatovski2024learn}. As this sensitivity decreases, the same preference pairs receive progressively smaller gradient weight, weakening their contribution to subsequent policy updates. Importantly, the initial choice of the coefficient $\beta$ also affects the sensitivity of the logistic loss term and can therefore lead to different optimization trajectories and downstream performance~\citep{rafailov2023direct}. These observations motivate a principled mechanism for characterizing and regulating variations in the training state.

In this paper, we analyze the DPO objective from a loss-level geometric perspective and identify a sigmoid factor that characterizes the local sensitivity of the objective during training. This factor reflects both the evolving optimization state and the effect of the coefficient $\beta$, making it a useful quantity for diagnosing and regulating DPO training. Based on this observation, we propose \textbf{Learning-Signal-Controlled Direct Preference Optimization (LSC-DPO)}, which dynamically regulates the learning signal to keep it close to a target level during preference optimization. We further derive a local stability condition for tracking the desired learning-signal regime. Empirically, LSC-DPO achieves strong performance compared with existing preference-optimization methods \citep{rafailov2023direct, wu2024beta, lee2025kl, park2024disentangling, azar2024general, ethayarajh2024kto, xu2024contrastive, hong2024orpo, zhao2023slic, meng2024simpo} across standard alignment benchmarks and preference-optimization experiments. We further use the learning-signal view to analyze sensitivity to the initial choice of $\beta$ and derive a signal-budget compensation rule that reduces initialization-induced performance degradation and improves robustness across different initial coefficients.

Our contributions are summarized in three aspects. (1) We provide a loss-level geometric characterization of DPO, showing that the sigmoid factor serves as a learning signal for tracking the local sensitivity of the objective and the evolving training state. (2) We propose \textbf{LSC-DPO}, a learning-signal-controlled preference optimization method that dynamically adjusts the DPO coefficient according to the evolving training state. We further derive a local stability condition that characterizes stable tracking around the target learning-signal regime. Empirically, LSC-DPO achieves strong performance compared with existing preference-optimization baselines across standard alignment benchmarks. (3) We investigate the relationship between the initial coefficient and learning-signal dynamics, and derive a signal-budget compensation rule that reduces initialization-induced performance degradation and improves robustness across different initial coefficients.

\section{Preliminaries}
In RLHF, a policy is optimized against a reward function while being
regularized toward a reference policy~\citep{christiano2017deep, jaques2017sequence,  jaques2020human}:
\begin{equation}
	\max_{\pi_\theta}
	\;
	\mathbb{E}_{x\sim\mathcal{D},\,y\sim\pi_\theta(\cdot|x)}
	\left[r(x,y)\right]
	-
	\beta\,
	\mathbb{E}_{x\sim\mathcal{D}}
	\left[
	\mathbb{D}_{\mathrm{KL}}
	\left(
	\pi_\theta(\cdot|x)
	\,\|\, 
	\pi_{\mathrm{ref}}(\cdot|x)
	\right)
	\right],
\end{equation}
where $\pi_\theta$ denotes the policy and $\pi_{\mathrm{ref}}$ denotes the reference model. DPO~\citep{rafailov2023direct} uses the closed-form solution~\citep{ziegler2019fine,peng2019advantage,peters2007reinforcement,go2023aligning,rafailov2023direct} to reparameterize the reward function as:
\begin{equation}
	r(x,y)
	=
	\beta \log
	\frac{\pi_{\theta}(y\mid x)}{\pi_{\mathrm{ref}}(y\mid x)}
	+
	\beta \log Z(x),
\end{equation}
where $Z(x)$ is the partition function. Substituting the reparameterized reward into the Bradley--Terry model~\citep{bradley1952rank} yields the DPO objective:
\begin{equation}
	\begin{aligned}
		\mathcal{L}_{\mathrm{DPO}}(\pi_{\theta};\pi_{\mathrm{ref}})
		=
		-\mathbb{E}_{(x,y_{+},y_{-})}
		\Bigg[
		\log \sigma
		\Bigg(
		&\beta \log
		\frac{\pi_{\theta}(y_{+}\mid x)}
		{\pi_{\mathrm{ref}}(y_{+}\mid x)}
		-
		\beta \log
		\frac{\pi_{\theta}(y_{-}\mid x)}
		{\pi_{\mathrm{ref}}(y_{-}\mid x)}
		\Bigg)
		\Bigg],
	\end{aligned}
\end{equation}
where $y_+$ and $y_-$ denote the chosen and rejected responses, respectively. The DPO objective depends on the difference between the policy log-ratio of the chosen and rejected responses and the corresponding reference model log-ratio. We define this difference as the preference margin:
\begin{equation}
	m=\log\frac{\pi_{\theta}(y_{+}\mid x)}{\pi_{\theta}(y_{-}\mid x)}-\log\frac{\pi_{\mathrm{ref}}(y_{+}\mid x)}{\pi_{\mathrm{ref}}(y_{-}\mid x)} .
\end{equation}
This margin measures how much more the policy favors the chosen response over the rejected response, relative to the reference model.  We further define the scaled preference margin as
\begin{equation}
	z = \beta m .
	\label{eq:scaled_margin}
\end{equation}
Using the margin formulation above, the DPO objective for a single preference pair can be written compactly as
\begin{equation}
	\ell_{\mathrm{DPO}}
	=
	-\log\sigma(\beta m)
	=
	-\log\sigma(z).
	\label{eq:dpo_scaled_loss}
\end{equation}
Thus, the optimization behavior of DPO depends on the scaled margin $z=\beta m$, which couples the evolving preference margin with the coefficient $\beta$.

\section{Learning-Signal-Controlled Direct Preference Optimization}
In this section, we introduce Learning-Signal-Controlled Direct Preference Optimization (LSC-DPO). We first identify the sigmoid factor $\sigma(-z)$ in the DPO objective as a learning signal that reflects the local sensitivity of the DPO loss during training and illustrate how different initial choices of $\beta$ affect the evolution of this signal. We then use this signal to characterize the training state of DPO and derive an update rule that regulates the signal during preference optimization. We analyze the stability of the resulting dynamics in log-space.

\subsection{Learning Signal in DPO}
\label{sec:effective_signal}
To understand the training behavior of DPO, we analyze the local structure of the DPO objective in Eq.~\ref{eq:dpo_scaled_loss}. The gradient can be
written as
\begin{equation}
	\nabla_\theta \ell
	=
	-\sigma(-z_\theta)\nabla_\theta z_\theta .
	\label{eq:dpo_gradient_z}
\end{equation}
Taking another derivative gives the Hessian of the DPO objective:
\begin{equation}
	\nabla_\theta^2 \ell
	=
	\sigma(z_\theta)\sigma(-z_\theta)
	\nabla_\theta z_\theta\nabla_\theta z_\theta^\top
	-
	\sigma(-z_\theta)\nabla_\theta^2 z_\theta .
	\label{eq:dpo_hessian_z}
\end{equation}
\begin{wrapfigure}{r}{0.48\textwidth}
	\centering
	\vspace{-0.5em}
	\includegraphics[width=0.48\textwidth]{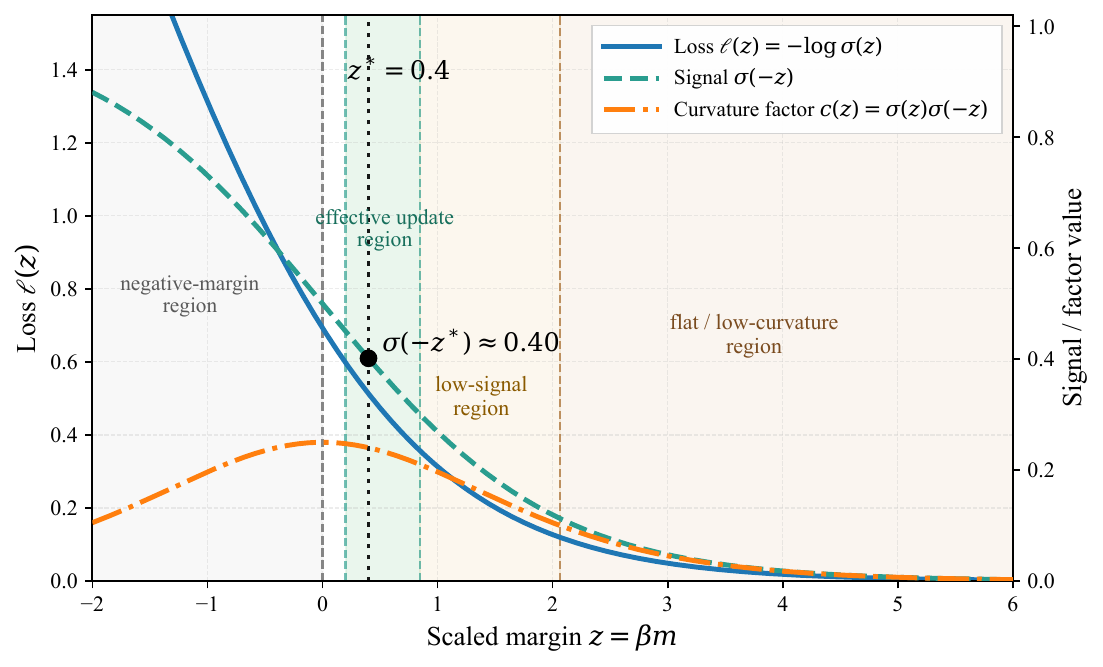}
	\caption{
		Illustration of the DPO loss geometry as a function of the scaled margin
		$z=\beta m$. The learning signal $\sigma(-z)$ and curvature factor
		$c(z)=\sigma(z)\sigma(-z)$ decrease as $z$ grows.
	}
	\label{fig:signal_geometry}
	\vspace{-1.0em}
\end{wrapfigure}
Equations~\ref{eq:dpo_gradient_z} and~\ref{eq:dpo_hessian_z} show that $\sigma(-z_\theta)$ plays a central role in the training dynamics of DPO. From the gradient perspective, $\sigma(-z_\theta)$ directly measures the first-order sensitivity of the DPO loss. From the curvature perspective, the Hessian contains two scaled terms: the outer-product term is weighted by $\sigma(z_\theta)\sigma(-z_\theta)$, while the second-order term is weighted by $\sigma(-z_\theta)$. During training, $z_\theta$ typically increases for preference pairs that become better aligned with the preference ordering. As $z_\theta$ grows, the sigmoid factor $\sigma(-z_\theta)$ decreases. Since these factors scale the Hessian terms that determine local curvature, their decrease indicates reduced local sensitivity of the objective and a weaker preference-learning tendency. Figure~\ref{fig:signal_geometry} illustrates this loss geometry.

We therefore use $\sigma(-z_\theta)$ as a learning signal for DPO training. When this signal is relatively large, the DPO loss remains more sensitive to changes in the scaled margin, corresponding to a high-signal regime. As $z_\theta$ increases, $\sigma(-z_\theta)$ decreases and the loss becomes progressively less sensitive to further changes in the scaled margin. We refer to this as a low-signal regime, which progressively approaches a lower-curvature region as $z_\theta$ grows further. Figures~\ref{fig:qt_fixed_beta_harmless} and~\ref{fig:qt_fixed_beta_helpful} further show that $\sigma(-z_\theta)$ evolves systematically over training. It remains relatively high during the early stage and gradually decreases before stabilizing at a lower level. This evolution is consistent with our interpretation of $\sigma(-z_\theta)$ as reflecting the local responsiveness of the DPO loss during training. The figures further show that different choices of the DPO coefficient lead to markedly different learning-signal trajectories $\sigma(-z)$. This observation suggests that the learning-signal dynamics are related to the coefficient setting and motivates our later analysis of coefficient initialization in Section~\ref{sec:init_beta}.

\begin{figure}
	\centering
	
	\begin{subfigure}[t]{0.48\linewidth}
		\centering
		\includegraphics[width=0.9\linewidth]{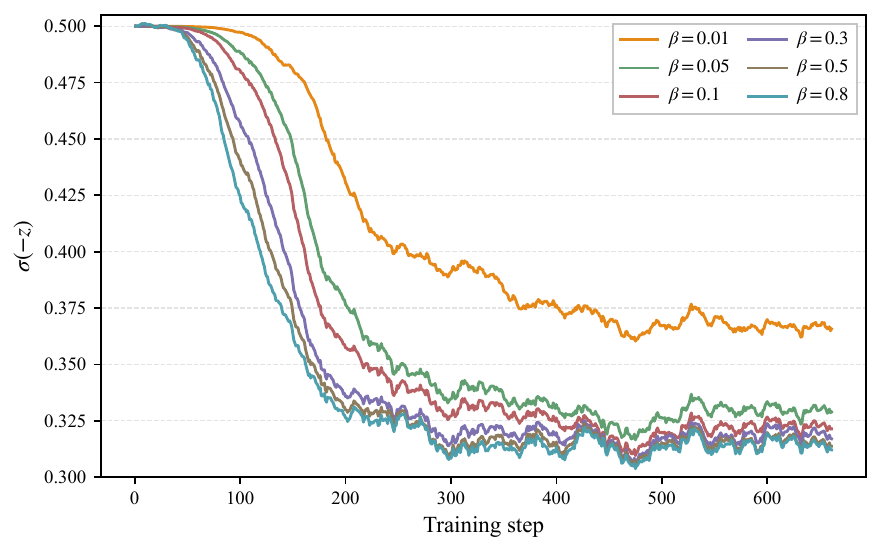}
		\caption{HH-harmless-base}
		\label{fig:qt_fixed_beta_harmless}
	\end{subfigure}
	\hfill
	\begin{subfigure}[t]{0.48\linewidth}
		\centering
		\includegraphics[width=0.9\linewidth]{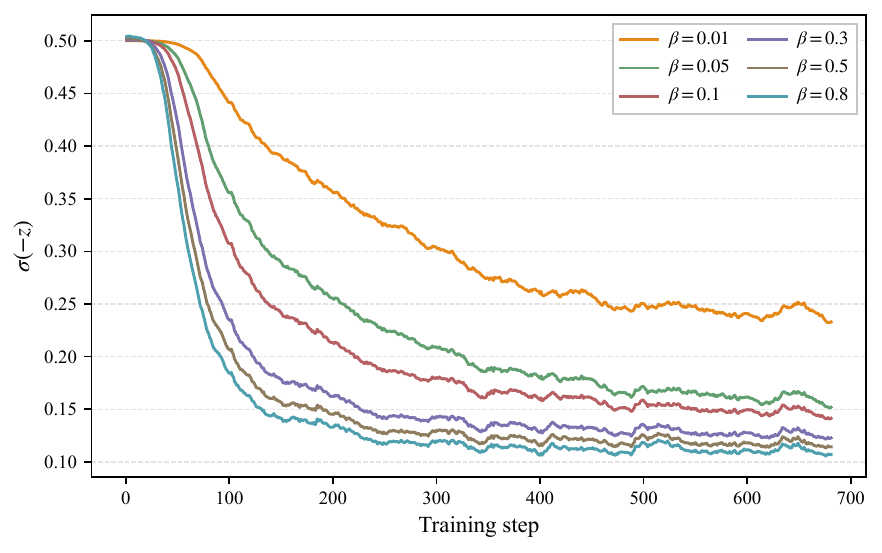}
		\caption{HH-helpful-base}
		\label{fig:qt_fixed_beta_helpful}
	\end{subfigure}
	
	\caption{
		$\sigma(-z)$ under fixed $\beta$ values during DPO training with Llama3-8B on two Anthropic-HH subsets: (a) harmless-base and (b) helpful-base. EMA-smoothed trajectories are shown for readability; raw trajectories are provided in Appendix~\ref{dpo_qt}.
	}
	\label{fig:qt_fixed_beta}
\end{figure}

\subsection{Learning-Signal-Controlled Update}
\label{sec:lsc_update}
Since $\sigma(-z)$ reflects the learning signal in DPO, we aim to keep it near a target level rather than allowing it to decay rapidly during training. Because $\sigma(-z)$ decreases monotonically with the scaled margin $z$, regulating $z$ provides a direct way to control the signal. Since $z=\beta m$, we use $\beta$ as the control variable.

At training step $k$, we compute the per-example preference margins on the current batch and use their mean as the batch-level margin $m_k$. We then obtain the corresponding scaled margin:
\begin{equation}
	z_k = \beta_k m_k .
\end{equation}
Thus, the controller computes the signal from the scaled batch-mean margin $z_k$. We then measure its deviation from a target scaled margin $z^*$:
\begin{equation}
	u_k = z_k - z^* .
\end{equation}
The target $z^*$ determines the learning-signal regime that LSC-DPO tracks,
equivalently corresponding to the target signal level $\sigma(-z^*)$. LSC-DPO updates $\beta$ using the following multiplicative rule:
\begin{equation}
	\label{eq:beta-update-rule}
	\beta_{k+1} = \beta_k \exp(-\eta u_k) = \beta_k \exp\left[-\eta(\beta_k m_k-z^*)\right],
\end{equation}
where $\eta>0$ controls the adjustment rate. The multiplicative form preserves the positivity of $\beta$ and implements feedback control over the scaled
margin. In the positive-margin regime, when $z_k>z^*$, $\sigma(-z_k)$ is below the target level and the update decreases $\beta_k$, moving the scaled margin downward. Conversely, when $z_k<z^*$, the update increases $\beta_k$, moving the scaled margin upward toward the target. The complete procedure is provided in Algorithm~\ref{alg:lsc_dpo} in Appendix~\ref{app:algorithm}.

\subsection{Stability Analysis of the Update Process}
\label{sec:stability_analysis}
We analyze the LSC-DPO update in log-space to characterize controller-level
tracking of the target $z^*$, rather than convergence of the full policy
optimization problem. Starting from Eq.~\ref{eq:beta-update-rule}, let
\begin{equation}
x_k := \log \beta_k.
\label{eq:xk_def}
\end{equation}

Substituting $\beta_k=e^{x_k}$ into Eq.~\ref{eq:beta-update-rule} gives
\begin{equation}
x_{k+1} = x_k-\eta\left(e^{x_k}m_k-z^*\right).
\end{equation}

Since the margin statistic $m_k$ evolves during training, the induced system is non-autonomous. Here, $m_k$ denotes the realized batch-level margin at step $k$ and may depend on the current policy, minibatch sampling, and previous coefficient updates. Accordingly, our analysis is conditional on the realized margin sequence and characterizes controller-level tracking rather than joint convergence of the full policy-training dynamics. Conditioned on $m_k$, the step-dependent update map is
\begin{equation}
F_k(x) = x-\eta\left(e^x m_k-z^*\right).
\label{eq:Fk_def}
\end{equation}

\paragraph{Instantaneous equilibrium.}
For a realized $m_k$, the map $F_k$ admits an instantaneous fixed point
\begin{equation}
x_k^f = \log\frac{z^*}{m_k}, \qquad \beta_k^f = \frac{z^*}{m_k},
\label{eq:instantaneous_equilibrium}
\end{equation}
provided that $m_k>0$ and $z^*>0$. Because $m_k$ evolves during training, this equilibrium moves over time. We next characterize its local stability.

\paragraph{Local stability.}
Linearizing the step-dependent map $F_k$ around the instantaneous
equilibrium $x_k^f$ gives
\begin{equation}
	F_k'(x_k^f)=1-\eta z^*.
\end{equation}
Hence, the instantaneous equilibrium is locally asymptotically stable whenever
\begin{equation}
	0<\eta z^*<2.
	\label{eq:stability_condition_main}
\end{equation}
Notably, this condition is independent of the realized margin $m_k$. The formal proposition and proof are provided in Appendix~\ref{prof:local_stability}.

\paragraph{Bounded tracking.}
Because $m_k$ evolves during training, the corresponding equilibrium
$x_k^f$ also moves over time. Under local contraction and bounded
equilibrium drift, the resulting tracking error satisfies
\begin{equation}
	|x_k-x_k^f|
	\le
	\rho^{k-T}|x_T-x_T^f|
	+
	\frac{\bar{\delta}}{1-\rho},
	\qquad k\ge T,
	\label{eq:ultimate_bound_main}
\end{equation}
where $\rho\in(0,1)$ denotes a local contraction factor and $\bar{\delta}$ bounds the step-to-step motion of the instantaneous equilibrium. The first term decays geometrically, while the second captures the persistent error induced by the moving target. Thus, when the realized margin evolves sufficiently gradually, the controller
remains close to the target scaled-margin regime and keeps the learning signal near $\sigma(-z^*)$. The complete theorem, assumptions, and proof are provided in Appendix~\ref{proof:theorem}.

\section{Experiments}
In this section, we evaluate LSC-DPO on standard alignment benchmarks and analyze its training dynamics. We first compare LSC-DPO with representative preference-optimization baselines on AlpacaEval 2~\citep{dubois2024length} and MT-Bench~\citep{zheng2023judging}, and then study its performance and learning-signal dynamics on Anthropic-HH~\citep{bai2022training}. We further validate the proposed learning-signal control through optimization-dynamics analyses and comparisons with predetermined $\beta$ schedules. Finally, we investigate robustness to the initial DPO coefficient and the proposed signal-budget compensation rule. Additional analyses of target-signal selection and hyperparameters are provided in Appendix~\ref{ablation_analysis}, control stability in Appendix~\ref{stability_analysis}, and downstream capability evaluations in Appendix~\ref{downstream}.

\subsection{Experimental Setup}
\label{sec:exp_setup}

\paragraph{Models and training datasets.}
\label{sec:models_training_setup}
We evaluate on two open-source models, meta-llama/Meta-Llama-3-8B (Llama3-8B)~\citep{grattafiori2024llama} and Qwen/Qwen3-8B-Base (Qwen3-8B)~\citep{yang2025qwen3}. We first perform supervised fine-tuning on UltraChat-200k~\citep{ding2023enhancing}, and then perform preference optimization on UltraFeedback~\citep{cui2023ultrafeedback}. We additionally use the harmless-base and helpful-base subsets of Anthropic-HH~\citep{bai2022training} for controlled analyses. Full training settings are provided in Appendix~\ref{app:models_training_datasets}.

\paragraph{Evaluation benchmarks and baselines.}
\label{eval_benchmarks}
For the UltraFeedback experiments, we evaluate on AlpacaEval~2~\citep{dubois2024length} using win rate and length-controlled (LC) win rate, and on MT-Bench~\citep{zheng2023judging}. On Anthropic-HH, we report pairwise win rates on both harmless-base and helpful-base. We compare our method with representative preference-optimization baselines, including DPO~\citep{rafailov2023direct}, $\beta$-DPO~\citep{wu2024beta}, $\varepsilon$-DPO~\citep{lee2025kl}, R-DPO~\citep{park2024disentangling}, IPO~\citep{azar2024general}, KTO~\citep{ethayarajh2024kto}, CPO~\citep{xu2024contrastive}, ORPO~\citep{hong2024orpo}, SLiC-HF~\citep{zhao2023slic}, and SimPO~\citep{meng2024simpo}. For the controlled Anthropic-HH analyses, we focus on DPO, $\beta$-DPO, and $\varepsilon$-DPO as the most relevant coefficient-control baselines. Detailed evaluation protocols and baseline configurations are provided in Appendix~\ref{app:evaluation_benchmarks_baselines}.

\subsection{Main Results on Standard Alignment Benchmarks}
\begin{table}
	\centering
	\small
	\setlength{\tabcolsep}{4.5pt}
	\caption{
		Evaluation results on AlpacaEval 2 and MT-Bench.
		For AlpacaEval 2, we report both win rate (WR) and length-controlled win rate (LC). For MT-Bench, we report the average score on a 1--10 scale.
	}
	\label{tab:main_eval}
	\begin{tabular}{lcccccc}
		\toprule
		\multirow{3}{*}{\textbf{Method}}
		& \multicolumn{3}{c}{\textbf{Llama3-8B}}
		& \multicolumn{3}{c}{\textbf{Qwen3-8B}} \\
		\cmidrule(lr){2-4} \cmidrule(lr){5-7}
		& \multicolumn{2}{c}{\textbf{AlpacaEval 2}}
		& \textbf{MT-Bench}
		& \multicolumn{2}{c}{\textbf{AlpacaEval 2}}
		& \textbf{MT-Bench} \\
		\cmidrule(lr){2-3} \cmidrule(lr){4-4}
		\cmidrule(lr){5-6} \cmidrule(lr){7-7}
		& \textbf{LC (\%)} & \textbf{WR (\%)} & \textbf{Score}
		& \textbf{LC (\%)} & \textbf{WR (\%)} & \textbf{Score} \\
		\midrule
		SFT
		& 6.8  & 3.7  & 6.3
		& 8.5  & 4.8  & 6.9 \\
		
		DPO
		& 19.2 & 17.7 & 6.6
		& 17.2 & 13.1 & 7.6 \\
		
		$\varepsilon$-DPO
		& 19.6 & 17.6 & \textbf{6.9}
		& 19.4 & 14.0 & 7.6 \\
		
		$\beta$-DPO
		& 16.6 & 16.5 & 6.8
		& 16.4 & 10.3 & 7.2 \\
		
		CPO
		& 11.2 & 10.3 & 6.4
		& 16.5 & 14.5 & 7.6 \\
		
		IPO
		& 8.6  & 8.5  & 6.6
		& 14.9 & 9.2  & 7.4 \\
		
		KTO
		& 16.9 & 15.1 & 6.8
		& 17.2 & 11.3 & 7.4 \\
		
		ORPO
		& 6.4  & 5.3  & 5.5
		& 9.4  & 6.8  & 6.9 \\
		
		R-DPO
		& 19.5 & 18.2 & 6.8
		& 19.3 & 13.2 & 7.6 \\
		
		SLiC-HF
		& 10.1 & 8.7  & 5.9
		& 12.5 & 10.5 & 7.1 \\
		
		SimPO
		& 19.7 & 19.2 & \textbf{6.9}
		& 16.8 & 10.9 & 7.4 \\
		
		\midrule
		\textbf{LSC-DPO}
		& \textbf{22.7} & \textbf{21.1} & \textbf{6.9}
		& \textbf{22.1} & \textbf{16.1} & \textbf{7.8} \\
		\bottomrule
	\end{tabular}
\end{table}
\paragraph{LSC-DPO achieves the strongest overall performance.}
Table~\ref{tab:main_eval} reports that LSC-DPO achieves the best AlpacaEval 2 performance on both Llama3-8B and Qwen3-8B, while attaining top or tied-top MT-Bench scores. Compared with DPO, LSC-DPO improves the length-controlled win rate from 19.2 to 22.7 on Llama3-8B and from 17.2 to 22.1 on Qwen3-8B, with similar gains in standard win rate. It also consistently outperforms $\beta$-DPO and $\varepsilon$-DPO and remains competitive with broader preference-optimization baselines, including R-DPO, KTO, CPO, ORPO, SLiC-HF, IPO, and SimPO. We next use Anthropic-HH for controlled comparisons of coefficient-regulation methods and their learning-signal behavior.

\begin{figure}
	\centering
	\begin{subfigure}{0.48\textwidth}
		\centering
		\includegraphics[width=0.8\linewidth]{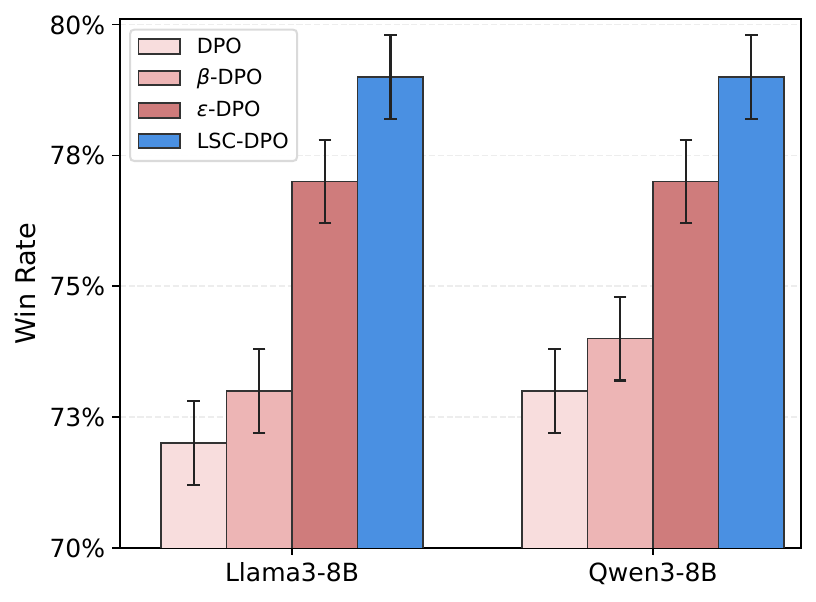}
		\caption{HH-harmless-base}
		\label{fig:performance_harmless}
	\end{subfigure}
	\hfill
	\begin{subfigure}{0.48\textwidth}
		\centering
		\includegraphics[width=0.8\linewidth]{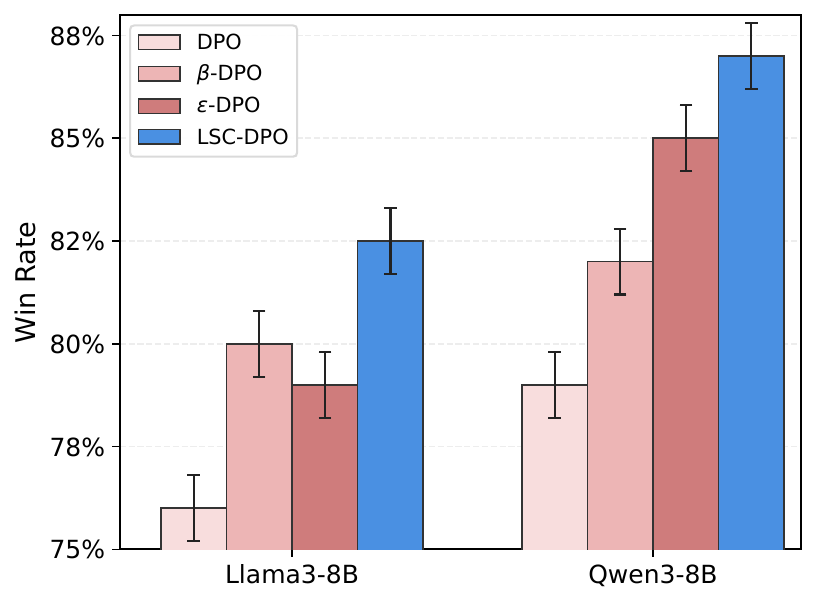}
		\caption{HH-helpful-base}
		\label{fig:performance_helpful}
	\end{subfigure}
	
	\caption{
		Performance comparison on Anthropic-HH. We report win rates of DPO, $\beta$-DPO, $\varepsilon$-DPO, and LSC-DPO using Llama3-8B and Qwen3-8B. Error bars denote approximate 95\% confidence intervals computed from prompt-level GPT-4 judgment results.
	}
	\label{fig:hh_performance}
\end{figure}

\begin{figure}
	\centering
	
	\begin{subfigure}{0.48\textwidth}
		\centering
		\includegraphics[width=0.9\linewidth]{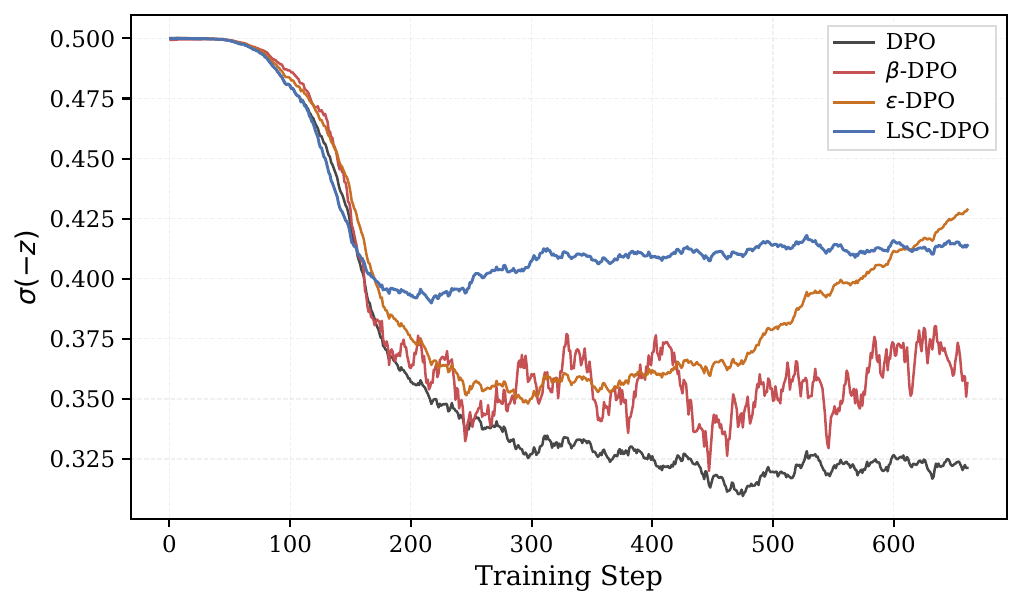}
		\caption{harmless-base}
		\label{fig:qt_harmless}
	\end{subfigure}
	\hfill
	\begin{subfigure}{0.48\textwidth}
		\centering
		\includegraphics[width=0.9\linewidth]{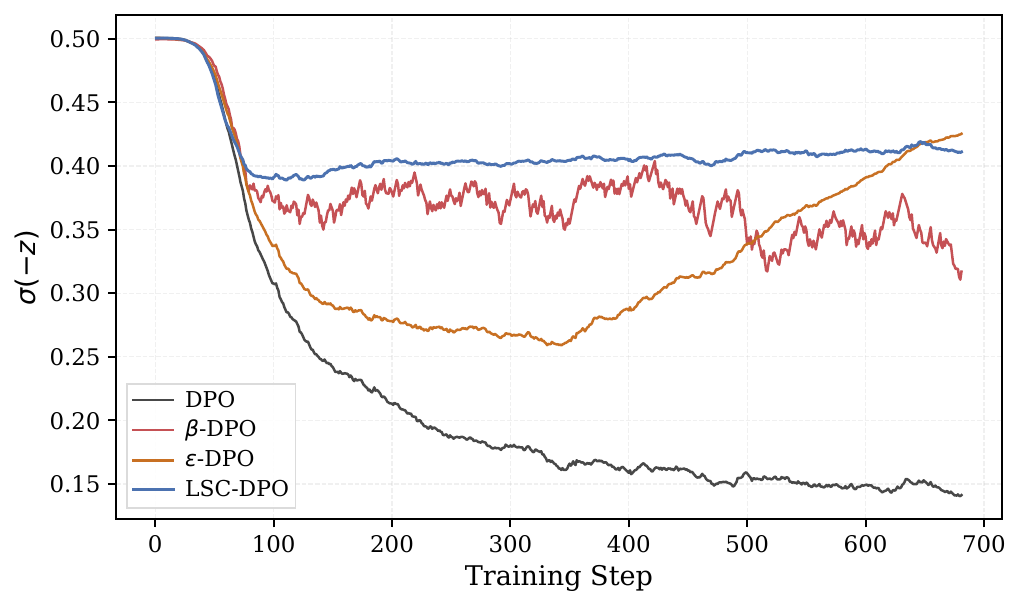}
		\caption{helpful-base}
		\label{fig:qt_helpful}
	\end{subfigure}
	
	\caption{
		Learning signal $\sigma(-z)$ on Anthropic-HH using Llama3-8B. The corresponding raw curves and Qwen results are provided in Appendix~\ref{app:additional_dynamics}.
	}
	\label{fig:hh_qt_llama}
\end{figure}
\paragraph{Overall performance on Anthropic-HH.}
Figures~\ref{fig:performance_harmless} and~\ref{fig:performance_helpful} show that LSC-DPO achieves the best pairwise win rates across both model families and Anthropic-HH subsets, reaching approximately 79\% on harmless-base and 82--87\% on helpful-base. It consistently outperforms DPO, $\beta$-DPO, and $\varepsilon$-DPO, further supporting the effectiveness of learning-signal control under controlled preference-optimization settings.

\paragraph{Learning-signal dynamics.}
\label{learning_signal_analysis}
Figures~\ref{fig:qt_harmless} and~\ref{fig:qt_helpful} show that DPO exhibits a sustained signal decay and eventually moves toward a low-signal regime, $\beta$-DPO shows larger fluctuations, and $\varepsilon$-DPO partially recovers after an initial decline. In contrast, LSC-DPO maintains a higher and more stable signal near its target regime across most of training. These relative dynamics are consistent with the performance trends in Figure~\ref{fig:hh_performance}; on helpful-base, for instance, $\varepsilon$-DPO maintains a lower signal than $\beta$-DPO for much of training before recovering later, aligning with its lower win rate. Overall, these results support the intended learning-signal regulation behavior of LSC-DPO.

\subsection{Mechanistic Validation of Learning-Signal Control}
\label{sec:mechanistic_validation}
\begin{table}[t]
	\caption{
		Optimization diagnostics before and after LSC-DPO control activation.
		We report the mean curvature factor and mean pre-clipping gradient norm.
	}
	\label{tab:mechanistic_diagnostics}
	\begin{center}
		\begin{tabular*}{\textwidth}{@{\extracolsep{\fill}}lcccc@{}}
			\toprule
			& \multicolumn{2}{c}{Before control (steps 1--83)}
			& \multicolumn{2}{c}{After control (steps 84--477)} \\
			\cmidrule(lr){2-3}\cmidrule(lr){4-5}
			Method
			& Curvature
			& Grad. norm
			& Curvature
			& Grad. norm \\
			\midrule
			DPO
			& 0.24993
			& 29.861
			& 0.23724
			& 90.330 \\
			
			\textbf{LSC-DPO}
			& 0.24993
			& 29.896
			& \textbf{0.24371}
			& 73.373 \\
			\bottomrule
		\end{tabular*}
	\end{center}
\end{table}
\begin{table}[t]
	\caption{Comparison with predetermined $\beta$ schedules.}
	\label{tab:beta_schedule_comparison}
	\centering
	\small
	\renewcommand{\arraystretch}{0.92}
	\begin{tabular*}{0.68\textwidth}
		{@{\extracolsep{\fill}}lcc@{}}
		\toprule
		Method & LC (\%) & WR (\%) \\
		\midrule
		DPO                          & 19.2 & 17.7 \\
		Linear $\beta$ schedule      & 19.7 & 18.4 \\
		Exponential $\beta$ schedule & 20.3 & 18.5 \\
		Cosine $\beta$ schedule      & 19.1 & 18.1 \\
		\textbf{LSC-DPO}             & \textbf{22.7} & \textbf{21.1} \\
		\bottomrule
	\end{tabular*}
\end{table}
\paragraph{Learning signal and realized optimization dynamics.}
Table~\ref{tab:mechanistic_diagnostics} shows that DPO and LSC-DPO are nearly indistinguishable before control activation: both have the same mean curvature factor and almost identical pre-clipping gradient norms. After control activation, LSC-DPO maintains a larger curvature factor than DPO. The slower reduction of this quantity indicates that LSC-DPO keeps the objective farther from the flatter, low-curvature regime. At the same time, its pre-clipping gradient norm is lower, showing that this increased loss-level responsiveness is not obtained by simply amplifying the overall parameter-gradient magnitude. Together, these results support our
interpretation of $\sigma(-z)$ as a meaningful loss-level training signal. Full training trajectories are provided in Appendix~\ref{app:optimization_dynamics_loss}.

\paragraph{Feedback versus predetermined $\beta$ schedules.}
Table~\ref{tab:beta_schedule_comparison} compares LSC-DPO with several predetermined schedules. Linear, exponential, and cosine schedules yield only modest changes relative to DPO, whereas LSC-DPO achieves substantially higher LC and WR. This suggests that the improvement cannot be explained solely by using a time-varying $\beta$ schedule and supports the role of state-dependent learning-signal feedback in LSC-DPO.

\subsection{Initialization Robustness and Signal Compensation}
\label{sec:init_beta}
\begin{figure}[t]
	\centering
	\includegraphics[width=0.6\linewidth]{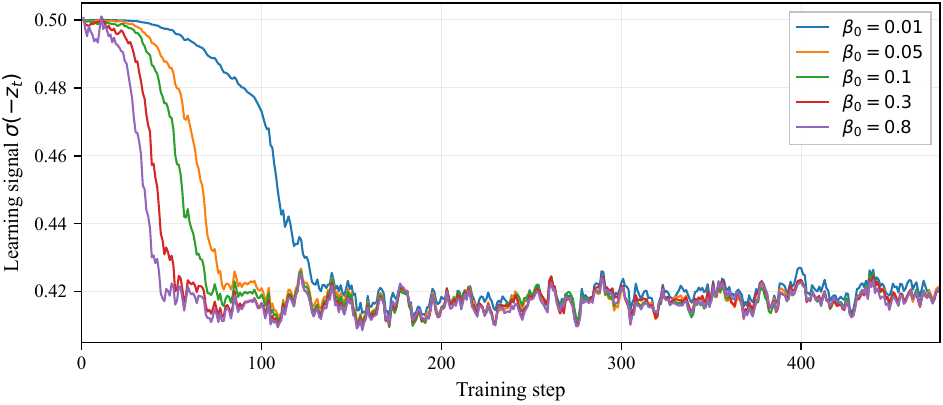}
	\caption{
		Learning-signal trajectories under different initial $\beta_0$.
		Curves are EMA-smoothed.
	}
	\label{fig:init_beta_signal}
\end{figure}
\paragraph{Initialization sensitivity.}
Table~\ref{tab:init_beta_performance} shows substantial sensitivity to $\beta_0$ across all methods. LSC-DPO achieves the highest performance at every tested initialization and degrades less severely, but still exhibits a 7.5-point LC spread. Figure~\ref{fig:init_beta_signal} further shows how different $\beta_0$ values affect the learning-signal trajectory $\sigma(-z)$. Although LSC-DPO eventually regulates the trajectories toward a similar controlled level, larger $\beta_0$ causes a faster decay of $\sigma(-z)$ during the early stage of training.

\begin{table}[t]
	\caption{
		Initialization sensitivity on AlpacaEval 2.
		We report length-controlled (LC) win rate (\%).
		$\dagger$ denotes the reference setting. Full LC/WR results are provided in Appendix~\ref{app:init_full_results}.
	}
	\label{tab:init_beta_performance}
	\centering
	\small
	\setlength{\tabcolsep}{4.5pt}
	\renewcommand{\arraystretch}{0.90}
	\setlength{\heavyrulewidth}{0.9pt}
	
	\begin{tabular*}{0.82\textwidth}
		{@{\extracolsep{\fill}}lccccc@{}}
		\toprule
		& \multicolumn{5}{c}{Initial coefficient $\beta_0$} \\
		\cmidrule(lr){2-6}
		Method
		& 0.01$^\dagger$
		& 0.05
		& 0.10
		& 0.30
		& 0.80 \\
		\midrule
		DPO
		& 19.2 & 10.4 & 8.7 & 8.3 & 6.6 \\
		
		$\beta$-DPO
		& 16.6 & 10.5 & 10.8 & 7.3 & 6.9 \\
		
		$\varepsilon$-DPO
		& 19.6 & 15.6 & 12.9 & 8.8 & 6.9 \\
		
		\textbf{LSC-DPO}
		& \textbf{22.7}
		& \textbf{19.6}
		& \textbf{19.1}
		& \textbf{15.3}
		& \textbf{15.2} \\
		\bottomrule
	\end{tabular*}
\end{table}

\paragraph{Signal-budget compensation.}
We hypothesize that the residual performance variation across different initial coefficients is related to differences in the cumulative learning signal during the transient stage. For an arbitrary anchor run $a$, we define the cumulative signal deficit as
\begin{equation}
	D_H(\beta_0)
	=
	\sum_{t=1}^{H}
	\left[
	\sigma(-z_t^a)
	-
	\sigma(-z_t^{(\beta_0)})
	\right],
	\label{eq:signal_deficit}
\end{equation}
where $H$ denotes the end of the transient phase. A positive $D_H(\beta_0)$ indicates that the run accumulates a smaller learning signal than the anchor during this stage. We compensate for this deficit over the remaining $T-H$ steps by setting
\begin{equation}
	\sigma(-z_{\mathrm{comp}}^*)
	=
	\sigma(-z_a^*)
	+
	\frac{D_H(\beta_0)}{T-H},
	\qquad
	z_{\mathrm{comp}}^*
	=
	\log
	\frac{1-\sigma(-z_{\mathrm{comp}}^*)}
	{\sigma(-z_{\mathrm{comp}}^*)},
	\label{eq:signal_compensation}
\end{equation}
where $z_{\mathrm{comp}}^*$ denotes the compensated target scaled margin corresponding to the compensated learning-signal target. Thus, a larger early-stage deficit leads to a higher compensated learning-signal target and, equivalently, a smaller $z^*$. Under ideal post-transient tracking, this rule matches the cumulative learning signal of the anchor run. We evaluate below whether this compensation also reduces the corresponding performance gap.

\begin{table}[t]
	\caption{
		Effect of full signal-budget compensation on AlpacaEval~2.
		We report length-controlled (LC) win rate (\%).
	}
	\label{tab:full_compensation}
	\centering
	\small
	\setlength{\tabcolsep}{5pt}
	\renewcommand{\arraystretch}{0.92}
	\setlength{\heavyrulewidth}{0.9pt}
	
	\begin{tabular*}{0.86\textwidth}
		{@{\extracolsep{\fill}}ccccc@{}}
		\toprule
		& \multicolumn{2}{c}{Baseline}
		& \multicolumn{2}{c}{Full compensation} \\
		\cmidrule(lr){2-3}
		\cmidrule(lr){4-5}
		
		$\beta_0$
		& $z^*$
		& LC
		& $z_{\mathrm{comp}}^*$
		& LC $\;(\Delta)$ \\
		\midrule
		
		0.05
		& 0.400
		& 19.6
		& 0.359
		& \textbf{21.5} \; (+1.9) \\
		
		0.10
		& 0.400
		& 19.1
		& 0.348
		& \textbf{21.1} \; (+2.0) \\
		
		0.30
		& 0.400
		& 15.3
		& 0.333
		& \textbf{20.4} \; (+5.1) \\
		
		0.80
		& 0.400
		& 15.2
		& 0.324
		& \textbf{19.7} \; (+4.5) \\
		
		\bottomrule
	\end{tabular*}
\end{table}
\paragraph{Compensation results.}
Table~\ref{tab:full_compensation} shows that signal-budget compensation improves the LC win rate for all tested non-anchor initializations, with gains of up to $5.1$ points, while reducing the LC spread from $4.4$ to $1.8$ points. These results indicate that adjusting the learning signal can partially compensate for the performance differences induced by coefficient initialization. Overall, the compensation substantially reduces, but does not eliminate, initialization sensitivity. Full results and additional analyses are provided in Appendix~\ref{app:signal_compensation_details}.

\section{Related Work}
\paragraph{Reinforcement learning from human feedback.}
Reinforcement learning from human feedback (RLHF) has become a central approach for aligning large language models with human preferences and values \citep{bai2022training, christiano2017deep, ouyang2022training, stiennon2020learning, ziegler2019fine}. A typical RLHF pipeline first collects preference comparisons, trains a reward model from human feedback \citep{chen2024odin, gao2023scaling, havrilla2024glore, luo2023wizardmath}, and then optimizes the policy with a reinforcement learning algorithm \citep{anthony2017thinking, schulman2017proximal}. Proximal policy optimization (PPO) is commonly used for this policy optimization stage \citep{schulman2017proximal}. Preference-based alignment has also been applied to improve helpfulness and safety, reduce toxic outputs, improve factuality, and control multiple objectives~\citep{dai2023safe, korbak2023pretraining, zheng2023click, tian2023fine, wang2024arithmetic}.

\paragraph{Direct preference optimization and variants.}
DPO~\citep{rafailov2023direct} has motivated a broad family of direct alignment methods that optimize language models from offline preference data without training an explicit reward model. Representative extensions include IPO~\citep{azar2024general}, KTO~\citep{ethayarajh2024kto}, RRHF~\citep{yuan2023rrhf}, R-DPO~\citep{park2024disentangling}, ODPO~\citep{amini2024direct}, CPO~\citep{xu2024contrastive}, ORPO~\citep{hong2024orpo}, SLiC-HF~\citep{zhao2023slic}, and SimPO~\citep{meng2024simpo}. These methods extend preference optimization through different objective designs, reference-model formulations, and robustness mechanisms.

\paragraph{DPO training dynamics and coefficient control.}
Recent work has studied DPO optimization from training-dynamics and
coefficient-control perspectives. Direct alignment methods can also exhibit overoptimization behavior as the policy moves away from the reference model~\citep{rafailov2024scaling}. A related line of work studies this training behavior and mitigates it by updating the reference policy during training~\citep{gorbatovski2024learn}. In the coefficient-control direction, $\beta$-DPO \citep{wu2024beta} updates $\beta$ using empirical reward-margin statistics, while $\varepsilon$-DPO~\citep{lee2025kl} perturbs the coefficient to reduce sensitivity to a fixed $\beta$. Related margin-adaptive approaches include AlphaDPO~\citep{wu2024alphadpo} and $\gamma$-PO~\citep{sun2025robust}.

\section{Conclusion}
\label{conclusion}
In this paper, we introduced Learning-Signal-Controlled Direct Preference Optimization (LSC-DPO), a feedback-control approach that regulates the learning signal of DPO during preference optimization.  By analyzing the DPO objective from a loss-level geometric perspective, we identified the sigmoid factor as a learning signal that characterizes the local sensitivity of the objective and used it to design a control rule that regulates DPO training dynamics around a target learning-signal level. Our log-space analysis shows that, under a simple local stability condition, the feedback update can stably track the target learning-signal regime. Empirical results further show that LSC-DPO maintains more sustained learning-signal dynamics and achieves stronger alignment performance. We further showed that different coefficient initializations are related to learning-signal dynamics and derived a signal-budget compensation rule that substantially reduces the resulting performance variation.

\paragraph{Limitations and Future Work.}
Although the best-performing learning-signal levels are similar across the model
families studied in this work, we do not yet provide a complete theory explaining why these target ranges emerge or when they transfer across models and datasets. In addition, although signal-budget compensation substantially reduces initialization-induced performance variation, it does not completely eliminate the remaining performance gap. Future work may develop a more complete theory of target learning-signal selection and signal-budget dynamics, and extend learning-signal control beyond DPO to other preference-optimization objectives.

\subsection*{AI use statement}
Generative AI tools were used in this work to edit the manuscript for grammar and readability, to check the formatting and the bibliography against the original sources, and to assist with writing and debugging the training and evaluation scripts. Generative AI tools were not used to generate synthetic data sets, to develop the theoretical framework, or to formulate or prove the mathematical claims; the remaining tasks with required disclosure are not applicable to this work. All AI-assisted text was reviewed and revised by the authors, AI-assisted code was verified and tested by the authors, and every citation was checked against the original publication. We take responsibility for the final content of this work, including text, claims, and artifacts produced with the aid of generative AI.

\subsection*{Ethics statement}
This work studies preference optimization for language-model alignment. All experiments use publicly available datasets (UltraChat-200k, UltraFeedback, and Anthropic-HH) and open-weight base models under their respective licenses, and no human-subject data were collected. The Anthropic-HH harmless-base subset contains prompts designed to elicit harmful responses; we use it only to evaluate harmlessness. As with any preference-optimization method, LSC-DPO can in principle be used to align a model toward undesirable preferences; our method does not change this risk relative to DPO. We do not release any new dataset.

\subsection*{Reproducibility statement}
The complete LSC-DPO procedure is given in Algorithm~\ref{alg:lsc_dpo} (Appendix~\ref{app:algorithm}). Models, datasets, hyperparameters, and compute are listed in Appendix~\ref{app:models_training_datasets}, and the evaluation protocols, judge models, and decoding settings in Appendix~\ref{app:evaluation_benchmarks_baselines}. The assumptions and full proofs of the stability results are in Appendices~\ref{appendix:proof} and~\ref{appendix:drift_condition}. The predetermined $\beta$ schedules and the signal-budget compensation targets are specified in Appendices~\ref{app:beta_schedule_comparison} and~\ref{app:signal_compensation_details}.

\bibliography{lsc_dpo_arxiv}

@article{ouyang2022training,
  title={Training language models to follow instructions with human feedback},
  author={Ouyang, Long and Wu, Jeffrey and Jiang, Xu and Almeida, Diogo and Wainwright, Carroll and Mishkin, Pamela and Zhang, Chong and Agarwal, Sandhini and Slama, Katarina and Ray, Alex and others},
  journal={Advances in Neural Information Processing Systems},
  volume={35},
  pages={27730--27744},
  year={2022}
}

@article{stiennon2020learning,
  title={Learning to summarize with human feedback},
  author={Stiennon, Nisan and Ouyang, Long and Wu, Jeffrey and Ziegler, Daniel and Lowe, Ryan and Voss, Chelsea and Radford, Alec and Amodei, Dario and Christiano, Paul F},
  journal={Advances in Neural Information Processing Systems},
  volume={33},
  pages={3008--3021},
  year={2020}
}

@article{christiano2017deep,
  title={Deep reinforcement learning from human preferences},
  author={Christiano, Paul F and Leike, Jan and Brown, Tom and Martic, Miljan and Legg, Shane and Amodei, Dario},
  journal={Advances in Neural Information Processing Systems},
  volume={30},
  year={2017}
}

@inproceedings{jaques2017sequence,
  title={Sequence tutor: Conservative fine-tuning of sequence generation models with {KL}-control},
  author={Jaques, Natasha and Gu, Shixiang and Bahdanau, Dzmitry and Hern{\'a}ndez-Lobato, Jos{\'e} Miguel and Turner, Richard E and Eck, Douglas},
  booktitle={International Conference on Machine Learning},
  pages={1645--1654},
  year={2017},
  organization={PMLR}
}

@inproceedings{jaques2020human,
  title={Human-centric dialog training via offline reinforcement learning},
  author={Jaques, Natasha and Shen, Judy Hanwen and Ghandeharioun, Asma and Ferguson, Craig and Lapedriza, Agata and Jones, Noah and Gu, Shixiang and Picard, Rosalind},
  booktitle={Proceedings of the 2020 Conference on Empirical Methods in Natural Language Processing (EMNLP)},
  pages={3985--4003},
  year={2020}
}

@article{ziegler2019fine,
  title={Fine-tuning language models from human preferences},
  author={Ziegler, Daniel M and Stiennon, Nisan and Wu, Jeffrey and Brown, Tom B and Radford, Alec and Amodei, Dario and Christiano, Paul and Irving, Geoffrey},
  journal={arXiv preprint arXiv:1909.08593},
  year={2019}
}

@article{peng2019advantage,
  title={Advantage-weighted regression: Simple and scalable off-policy reinforcement learning},
  author={Peng, Xue Bin and Kumar, Aviral and Zhang, Grace and Levine, Sergey},
  journal={arXiv preprint arXiv:1910.00177},
  year={2019}
}

@inproceedings{peters2007reinforcement,
  title={Reinforcement learning by reward-weighted regression for operational space control},
  author={Peters, Jan and Schaal, Stefan},
  booktitle={Proceedings of the 24th International Conference on Machine Learning},
  pages={745--750},
  year={2007}
}

@inproceedings{go2023aligning,
  title={Aligning language models with preferences through f-divergence minimization},
  author={Go, Dongyoung and Korbak, Tomasz and Kruszewski, Germ{\'a}n and Rozen, Jos and Ryu, Nahyeon and Dymetman, Marc},
  booktitle={International Conference on Machine Learning},
  volume={202},
  pages={11546--11583},
  year={2023},
  organization={PMLR}
}

@article{rafailov2023direct,
  title={Direct preference optimization: {Y}our language model is secretly a reward model},
  author={Rafailov, Rafael and Sharma, Archit and Mitchell, Eric and Manning, Christopher D and Ermon, Stefano and Finn, Chelsea},
  journal={Advances in Neural Information Processing Systems},
  volume={36},
  pages={53728--53741},
  year={2023}
}

@article{bradley1952rank,
  title={Rank analysis of incomplete block designs: {I}. {T}he method of paired comparisons},
  author={Bradley, Ralph Allan and Terry, Milton E},
  journal={Biometrika},
  volume={39},
  number={3/4},
  pages={324--345},
  year={1952},
  publisher={JSTOR}
}

@article{wu2024beta,
  title={$\beta$-{DPO}: Direct preference optimization with dynamic $\beta$},
  author={Wu, Junkang and Xie, Yuexiang and Yang, Zhengyi and Wu, Jiancan and Gao, Jinyang and Ding, Bolin and Wang, Xiang and He, Xiangnan},
  journal={Advances in Neural Information Processing Systems},
  volume={37},
  pages={129944--129966},
  year={2024}
}

@article{lee2025kl,
  title={{KL} Penalty Control via Perturbation for Direct Preference Optimization},
  author={Lee, Sangkyu and Han, Janghoon and Song, Hosung and Choi, Stanley Jungkyu and Lee, Honglak and Yu, Youngjae},
  journal={Advances in Neural Information Processing Systems},
  volume={38},
  pages={168095--168121},
  year={2025}
}

@article{schulman2017proximal,
  title={Proximal policy optimization algorithms},
  author={Schulman, John and Wolski, Filip and Dhariwal, Prafulla and Radford, Alec and Klimov, Oleg},
  journal={arXiv preprint arXiv:1707.06347},
  year={2017}
}

@article{bai2022training,
  title={Training a helpful and harmless assistant with reinforcement learning from human feedback},
  author={Bai, Yuntao and Jones, Andy and Ndousse, Kamal and Askell, Amanda and Chen, Anna and DasSarma, Nova and Drain, Dawn and Fort, Stanislav and Ganguli, Deep and Henighan, Tom and others},
  journal={arXiv preprint arXiv:2204.05862},
  year={2022}
}

@inproceedings{chen2024odin,
  title={{ODIN}: Disentangled reward mitigates hacking in {RLHF}},
  author={Chen, Lichang and Zhu, Chen and Chen, Jiuhai and Soselia, Davit and Zhou, Tianyi and Goldstein, Tom and Huang, Heng and Shoeybi, Mohammad and Catanzaro, Bryan},
  booktitle={International Conference on Machine Learning},
  volume={235},
  pages={7935--7952},
  year={2024},
  organization={PMLR}
}

@inproceedings{gao2023scaling,
  title={Scaling laws for reward model overoptimization},
  author={Gao, Leo and Schulman, John and Hilton, Jacob},
  booktitle={International Conference on Machine Learning},
  pages={10835--10866},
  year={2023},
  organization={PMLR}
}

@article{havrilla2024glore,
  title={{GLoRe}: When, where, and how to improve {LLM} reasoning via global and local refinements},
  author={Havrilla, Alex and Raparthy, Sharath and Nalmpantis, Christoforus and Dwivedi-Yu, Jane and Zhuravinskyi, Maksym and Hambro, Eric and Raileanu, Roberta},
  journal={arXiv preprint arXiv:2402.10963},
  year={2024}
}

@article{luo2023wizardmath,
  title={{WizardMath}: Empowering mathematical reasoning for large language models via reinforced {Evol-Instruct}},
  author={Luo, Haipeng and Sun, Qingfeng and Xu, Can and Zhao, Pu and Lou, Jianguang and Tao, Chongyang and Geng, Xiubo and Lin, Qingwei and Chen, Shifeng and Zhang, Dongmei},
  journal={arXiv preprint arXiv:2308.09583},
  year={2023}
}

@article{anthony2017thinking,
  title={Thinking fast and slow with deep learning and tree search},
  author={Anthony, Thomas and Tian, Zheng and Barber, David},
  journal={Advances in Neural Information Processing Systems},
  volume={30},
  year={2017}
}

@inproceedings{azar2024general,
  title={A general theoretical paradigm to understand learning from human preferences},
  author={Azar, Mohammad Gheshlaghi and Guo, Zhaohan Daniel and Piot, Bilal and Munos, Remi and Rowland, Mark and Valko, Michal and Calandriello, Daniele},
  booktitle={International Conference on Artificial Intelligence and Statistics},
  pages={4447--4455},
  year={2024},
  organization={PMLR}
}

@inproceedings{ethayarajh2024kto,
  title={Model alignment as prospect theoretic optimization},
  author={Ethayarajh, Kawin and Xu, Winnie and Muennighoff, Niklas and Jurafsky, Dan and Kiela, Douwe},
  booktitle={International Conference on Machine Learning},
  volume={235},
  pages={12634--12651},
  year={2024},
  organization={PMLR}
}

@article{yuan2023rrhf,
  title={{RRHF}: Rank responses to align language models with human feedback without tears},
  author={Yuan, Zheng and Yuan, Hongyi and Tan, Chuanqi and Wang, Wei and Huang, Songfang and Huang, Fei},
  journal={arXiv preprint arXiv:2304.05302},
  year={2023}
}

@article{zhao2023slic,
  title={{SLiC-HF}: Sequence likelihood calibration with human feedback},
  author={Zhao, Yao and Joshi, Rishabh and Liu, Tianqi and Khalman, Misha and Saleh, Mohammad and Liu, Peter J},
  journal={arXiv preprint arXiv:2305.10425},
  year={2023}
}

@inproceedings{xu2024contrastive,
  title={Contrastive preference optimization: Pushing the boundaries of {LLM} performance in machine translation},
  author={Xu, Haoran and Sharaf, Amr and Chen, Yunmo and Tan, Weiting and Shen, Lingfeng and Van Durme, Benjamin and Murray, Kenton and Kim, Young Jin},
  booktitle={International Conference on Machine Learning},
  pages={55204--55224},
  year={2024},
  organization={PMLR}
}

@inproceedings{hong2024orpo,
  title={{ORPO}: Monolithic preference optimization without reference model},
  author={Hong, Jiwoo and Lee, Noah and Thorne, James},
  booktitle={Proceedings of the 2024 Conference on Empirical Methods in Natural Language Processing},
  pages={11170--11189},
  year={2024}
}

@article{meng2024simpo,
  title={{SimPO}: Simple preference optimization with a reference-free reward},
  author={Meng, Yu and Xia, Mengzhou and Chen, Danqi},
  journal={Advances in Neural Information Processing Systems},
  volume={37},
  pages={124198--124235},
  year={2024}
}

@inproceedings{park2024disentangling,
  title={Disentangling length from quality in direct preference optimization},
  author={Park, Ryan and Rafailov, Rafael and Ermon, Stefano and Finn, Chelsea},
  booktitle={Findings of the Association for Computational Linguistics: ACL 2024},
  pages={4998--5017},
  year={2024}
}

@article{gorbatovski2024learn,
  title={Learn your reference model for real good alignment},
  author={Gorbatovski, Alexey and Shaposhnikov, Boris and Malakhov, Alexey and Surnachev, Nikita and Aksenov, Yaroslav and Maksimov, Ian and Balagansky, Nikita and Gavrilov, Daniil},
  journal={arXiv preprint arXiv:2404.09656},
  year={2024}
}

@article{dubois2024length,
  title={Length-controlled {AlpacaEval}: A simple way to debias automatic evaluators},
  author={Dubois, Yann and Galambosi, Bal{\'a}zs and Liang, Percy and Hashimoto, Tatsunori B},
  journal={arXiv preprint arXiv:2404.04475},
  year={2024}
}

@article{zheng2023judging,
  title={Judging {LLM}-as-a-judge with {MT-Bench} and {Chatbot Arena}},
  author={Zheng, Lianmin and Chiang, Wei-Lin and Sheng, Ying and Zhuang, Siyuan and Wu, Zhanghao and Zhuang, Yonghao and Lin, Zi and Li, Zhuohan and Li, Dacheng and Xing, Eric and others},
  journal={Advances in Neural Information Processing Systems},
  volume={36},
  pages={46595--46623},
  year={2023}
}

@article{grattafiori2024llama,
  title={The {Llama 3} herd of models},
  author={Grattafiori, Aaron and Dubey, Abhimanyu and Jauhri, Abhinav and Pandey, Abhinav and Kadian, Abhishek and Al-Dahle, Ahmad and Letman, Aiesha and Mathur, Akhil and Schelten, Alan and Vaughan, Alex and others},
  journal={arXiv preprint arXiv:2407.21783},
  year={2024}
}

@article{yang2025qwen3,
  title={{Qwen3} technical report},
  author={Yang, An and Li, Anfeng and Yang, Baosong and Zhang, Beichen and Hui, Binyuan and Zheng, Bo and Yu, Bowen and Gao, Chang and Huang, Chengen and Lv, Chenxu and others},
  journal={arXiv preprint arXiv:2505.09388},
  year={2025}
}

@inproceedings{ding2023enhancing,
  title={Enhancing chat language models by scaling high-quality instructional conversations},
  author={Ding, Ning and Chen, Yulin and Xu, Bokai and Qin, Yujia and Hu, Shengding and Liu, Zhiyuan and Sun, Maosong and Zhou, Bowen},
  booktitle={Proceedings of the 2023 Conference on Empirical Methods in Natural Language Processing},
  pages={3029--3051},
  year={2023}
}

@article{cui2023ultrafeedback,
  title={{UltraFeedback}: Boosting language models with scaled {AI} feedback},
  author={Cui, Ganqu and Yuan, Lifan and Ding, Ning and Yao, Guanming and He, Bingxiang and Zhu, Wei and Ni, Yuan and Xie, Guotong and Xie, Ruobing and Lin, Yankai and others},
  journal={arXiv preprint arXiv:2310.01377},
  year={2023}
}

@article{li2024crowdsourced,
  title={From crowdsourced data to high-quality benchmarks: {Arena-Hard} and {BenchBuilder} pipeline},
  author={Li, Tianle and Chiang, Wei-Lin and Frick, Evan and Dunlap, Lisa and Wu, Tianhao and Zhu, Banghua and Gonzalez, Joseph E and Stoica, Ion},
  journal={arXiv preprint arXiv:2406.11939},
  year={2024}
}

@article{rafailov2024scaling,
  title={Scaling laws for reward model overoptimization in direct alignment algorithms},
  author={Rafailov, Rafael and Chittepu, Yaswanth and Park, Ryan and Sikchi, Harshit and Hejna, Joey and Knox, W Bradley and Finn, Chelsea and Niekum, Scott},
  journal={Advances in Neural Information Processing Systems},
  volume={37},
  pages={126207--126242},
  year={2024}
}

@inproceedings{amini2024direct,
  title={Direct preference optimization with an offset},
  author={Amini, Afra and Vieira, Tim and Cotterell, Ryan},
  booktitle={Findings of the Association for Computational Linguistics: ACL 2024},
  pages={9954--9972},
  year={2024}
}

@inproceedings{korbak2023pretraining,
  title={Pretraining language models with human preferences},
  author={Korbak, Tomasz and Shi, Kejian and Chen, Angelica and Bhalerao, Rasika Vinayak and Buckley, Christopher and Phang, Jason and Bowman, Samuel R and Perez, Ethan},
  booktitle={International Conference on Machine Learning},
  pages={17506--17533},
  year={2023},
  organization={PMLR}
}

@inproceedings{zheng2023click,
  title={{Click}: Controllable text generation with sequence likelihood contrastive learning},
  author={Zheng, Chujie and Ke, Pei and Zhang, Zheng and Huang, Minlie},
  booktitle={Findings of the Association for Computational Linguistics: ACL 2023},
  pages={1022--1040},
  year={2023}
}

@article{dai2023safe,
  title={Safe {RLHF}: Safe reinforcement learning from human feedback},
  author={Dai, Josef and Pan, Xuehai and Sun, Ruiyang and Ji, Jiaming and Xu, Xinbo and Liu, Mickel and Wang, Yizhou and Yang, Yaodong},
  journal={arXiv preprint arXiv:2310.12773},
  year={2023}
}

@inproceedings{tian2023fine,
  title={Fine-tuning language models for factuality},
  author={Tian, Katherine and Mitchell, Eric and Yao, Huaxiu and Manning, Christopher D and Finn, Chelsea},
  booktitle={The Twelfth International Conference on Learning Representations},
  year={2024}
}

@inproceedings{wang2024arithmetic,
  title={Arithmetic control of {LLMs} for diverse user preferences: Directional preference alignment with multi-objective rewards},
  author={Wang, Haoxiang and Lin, Yong and Xiong, Wei and Yang, Rui and Diao, Shizhe and Qiu, Shuang and Zhao, Han and Zhang, Tong},
  booktitle={Proceedings of the 62nd Annual Meeting of the Association for Computational Linguistics (Volume 1: Long Papers)},
  pages={8642--8655},
  year={2024}
}

@article{hendrycks2020measuring,
  title={Measuring massive multitask language understanding},
  author={Hendrycks, Dan and Burns, Collin and Basart, Steven and Zou, Andy and Mazeika, Mantas and Song, Dawn and Steinhardt, Jacob},
  journal={arXiv preprint arXiv:2009.03300},
  year={2020}
}

@article{clark2018think,
  title={Think you have solved question answering? {T}ry {ARC}, the {AI2} reasoning challenge},
  author={Clark, Peter and Cowhey, Isaac and Etzioni, Oren and Khot, Tushar and Sabharwal, Ashish and Schoenick, Carissa and Tafjord, Oyvind},
  journal={arXiv preprint arXiv:1803.05457},
  year={2018}
}

@inproceedings{zellers2019hellaswag,
  title={{HellaSwag}: Can a machine really finish your sentence?},
  author={Zellers, Rowan and Holtzman, Ari and Bisk, Yonatan and Farhadi, Ali and Choi, Yejin},
  booktitle={Proceedings of the 57th Annual Meeting of the Association for Computational Linguistics},
  pages={4791--4800},
  year={2019}
}

@inproceedings{lin2022truthfulqa,
  title={{TruthfulQA}: Measuring how models mimic human falsehoods},
  author={Lin, Stephanie and Hilton, Jacob and Evans, Owain},
  booktitle={Proceedings of the 60th Annual Meeting of the Association for Computational Linguistics (Volume 1: Long Papers)},
  pages={3214--3252},
  year={2022}
}

@article{sakaguchi2021winogrande,
  title={{WinoGrande}: An adversarial {W}inograd schema challenge at scale},
  author={Sakaguchi, Keisuke and Le Bras, Ronan and Bhagavatula, Chandra and Choi, Yejin},
  journal={Communications of the ACM},
  volume={64},
  number={9},
  pages={99--106},
  year={2021},
  doi={10.1145/3474381}
}

@article{cobbe2021training,
  title={Training verifiers to solve math word problems},
  author={Cobbe, Karl and Kosaraju, Vineet and Bavarian, Mohammad and Chen, Mark and Jun, Heewoo and Kaiser, Lukasz and Plappert, Matthias and Tworek, Jerry and Hilton, Jacob and Nakano, Reiichiro and others},
  journal={arXiv preprint arXiv:2110.14168},
  year={2021}
}

@inproceedings{sun2025robust,
  title={Robust preference optimization via dynamic target margins},
  author={Sun, Jie and Wu, Junkang and Wu, Jiancan and Zhu, Zhibo and Lu, Xingyu and Zhou, Jun and Ma, Lintao and Wang, Xiang},
  booktitle={Findings of the Association for Computational Linguistics: ACL 2025},
  pages={5399--5416},
  year={2025}
}

@article{wu2024alphadpo,
  title={{AlphaDPO}: Adaptive reward margin for direct preference optimization},
  author={Wu, Junkang and Wang, Xue and Yang, Zhengyi and Wu, Jiancan and Gao, Jinyang and Ding, Bolin and Wang, Xiang and He, Xiangnan},
  journal={arXiv preprint arXiv:2410.10148},
  year={2024}
}
\bibliographystyle{iclr2027_conference}

\appendix
\raggedbottom

\section{Learning Signal $\sigma(-z)$ under Fixed $\beta$}
\label{dpo_qt}
We report the raw per-step trajectories of the learning signal under fixed-$\beta$ DPO with different $\beta$ values. Figures~\ref{fig:qt_fixed_beta_raw_harmless} and~\ref{fig:qt_fixed_beta_raw_helpful} show that
$\sigma(-z)$ generally decreases during training on both Anthropic-HH subsets. 

\begin{figure}[H]
	\centering
	
	\begin{subfigure}[t]{0.48\linewidth}
		\centering
		\includegraphics[width=\linewidth]{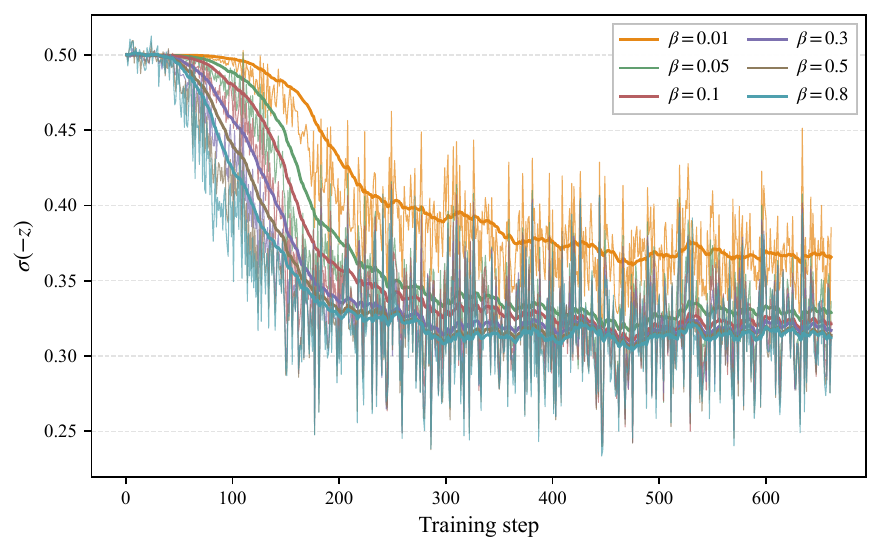}
		\caption{harmless-base}
		\label{fig:qt_fixed_beta_raw_harmless}
	\end{subfigure}
	\hfill
	\begin{subfigure}[t]{0.48\linewidth}
		\centering
		\includegraphics[width=\linewidth]{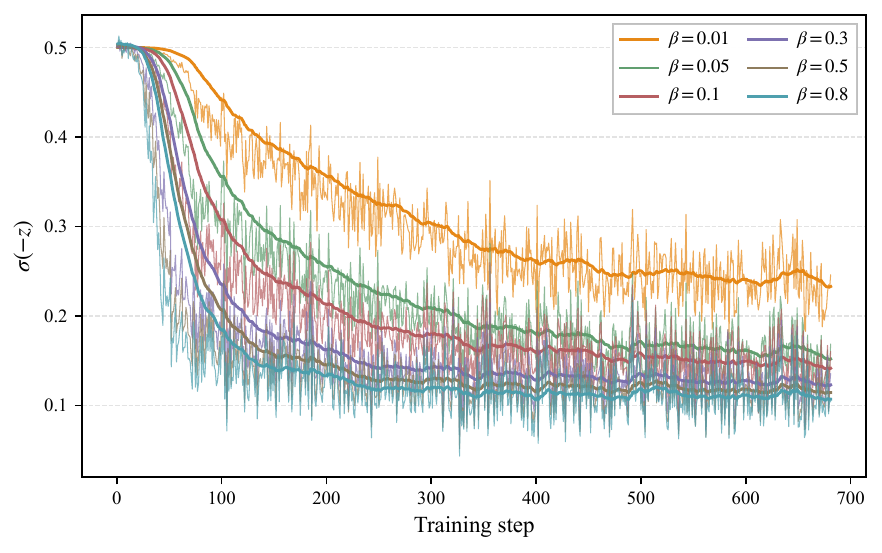}
		\caption{helpful-base}
		\label{fig:qt_fixed_beta_raw_helpful}
	\end{subfigure}
	
	\caption{Learning signal under DPO. Raw per-step trajectories of $\sigma(-z)$ are shown for different fixed $\beta$ values on Anthropic-HH: (a) harmless-base and (b) helpful-base.
	}
	\label{fig:qt_fixed_beta_raw}
\end{figure}

\section{LSC-DPO Algorithm}
\label{app:algorithm}
Algorithm~\ref{alg:lsc_dpo} summarizes the complete training procedure of LSC-DPO, including computation of the learning signal, controller activation, and the subsequent coefficient update described in Section~\ref{sec:lsc_update}.
\begin{algorithm}[H]
	\small
	\caption{Learning-Signal-Controlled Direct Preference Optimization}
	\label{alg:lsc_dpo}
	\begin{algorithmic}[1]
		\Require policy $\pi_\theta$, reference $\pi_{\mathrm{ref}}$, initial $\beta_0$, $z^*$, $\eta$, $\tau_{\mathrm{start}}$
		\State $\beta \leftarrow \beta_0$; control is inactive
		\While{not converged}
		\State Sample a batch of preference triples $(x,y_+,y_-)\sim\mathcal{D}$
		\State Compute margin $m$, scaled margin $z=\beta m$, and signal $\sigma(-z)$
		\State Activate control if $\sigma(-z)\le \tau_{\mathrm{start}}$
		\State If control is active, update $\beta \leftarrow \beta\exp[-\eta(z-z^*)]$
		\State Update $\pi_\theta$ using the DPO loss with the current $\beta$
		\EndWhile
		\State \Return aligned policy $\pi_\theta$
	\end{algorithmic}
\end{algorithm}

\section{Proofs for the Stability Analysis}
\label{appendix:proof}
In this appendix, we provide the formal statements and detailed proofs
for the local stability and bounded tracking results summarized in Section~\ref{sec:stability_analysis}. Recall that the log-space update is
\begin{equation}
	x_{k+1}=F_k(x_k),
	\qquad
	F_k(x):=x-\eta\left(e^x m_k-z^*\right),
	\label{eq:appendix_Fk_def}
\end{equation}
and that the instantaneous equilibrium is
\begin{equation}
	x_k^f=\log\frac{z^*}{m_k}.
	\label{eq:appendix_xkf}
\end{equation}

\subsection{Proof of Local Stability}
\label{prof:local_stability}
The formal local-stability result is stated as follows.
\begin{proposition}[Local stability]
	\label{prop:local-stability}
	Assume $m_k>0$ and $z^*>0$. The instantaneous equilibrium $x_k^f=\log(z^*/m_k)$ is locally asymptotically stable whenever
	\begin{equation}
		0<\eta z^*<2.
		\label{eq:stability_condition_append}
	\end{equation}
	The corresponding equilibrium linearization factor is
	\begin{equation}
		\rho_0 := |1-\eta z^*| < 1.
		\label{eq:rho_def}
	\end{equation}
\end{proposition}

This result follows by linearizing $F_k$ around $x_k^f$, where
\begin{equation}
	F_k'(x_k^f)=1-\eta z^*.
\end{equation}

\begin{proof}[Proof]
	Since $m_k>0$ and $z^*>0$, the instantaneous equilibrium is well defined and satisfies
	\begin{equation}
		e^{x_k^f}m_k=z^*.
	\end{equation}
	The derivative of the step-dependent map is
	\begin{equation}
		F_k'(x)
		=
		1-\eta e^x m_k.
	\end{equation}
	Evaluating this derivative at the equilibrium gives
	\begin{equation}
		F_k'(x_k^f)
		=
		1-\eta e^{x_k^f}m_k
		=
		1-\eta z^*.
	\end{equation}
	For a one-dimensional discrete-time dynamical system, a fixed point is locally asymptotically stable if the derivative at the fixed point has magnitude smaller than one. Hence,
	\begin{equation}
		|F_k'(x_k^f)|<1
		\quad\Longleftrightarrow\quad
		|1-\eta z^*|<1.
	\end{equation}
	Solving this inequality yields
	\begin{equation}
		0<\eta z^*<2.
	\end{equation}
	Under this condition, the equilibrium linearization factor is
	\begin{equation}
		\rho_0
		:=
		|F_k'(x_k^f)|
		=
		|1-\eta z^*|
		<1,
	\end{equation}
	which proves the proposition.
\end{proof}

\subsection{Bounded Tracking Theorem and Proof}
\label{proof:theorem}
\begin{theorem}[Bounded tracking]
	\label{thm:bounded_tracking}
	Assume the local stability condition in
	Eq.~\ref{eq:stability_condition_main} holds, and define
	\[
	\rho_0 := |1-\eta z^*| < 1.
	\]
	For any $\rho$ satisfying
	\[
	\rho_0 < \rho < 1,
	\]
	there exists a radius $r>0$ such that, for every $k$,
	\begin{equation}
		|F_k'(x)| \le \rho
		\qquad
		\text{whenever }
		|x-x_k^f| \le r.
		\label{eq:local_contraction_neighborhood}
	\end{equation}
	Suppose that there exists a step $T$ such that
	\begin{equation}
		|x_T-x_T^f| \le r,
		\label{eq:initial_tracking_neighborhood}
	\end{equation}
	and that, for some constant $\bar{\delta}\ge 0$ and all $k\ge T$,
	\begin{equation}
		|x_{k+1}^f-x_k^f| \le \bar{\delta},
		\qquad
		\bar{\delta} \le (1-\rho)r.
		\label{eq:delta_assumption}
	\end{equation}
	Then the iterates remain within the local contraction neighborhood,
	\[
	|x_k-x_k^f| \le r,
	\qquad
	\forall k\ge T,
	\]
	and satisfy
	\begin{equation}
		|x_k-x_k^f|
		\le
		\rho^{k-T}|x_T-x_T^f|
		+
		\frac{\bar{\delta}}{1-\rho},
		\qquad
		\forall k\ge T.
		\label{eq:ultimate_bound_appendix}
	\end{equation}
\end{theorem}
The complete proof is given below. An interpretation of the resulting tracking bound is provided in Appendix~\ref{sec:tracking_bound_interpretation}.

We define the tracking error by
\begin{equation}
	e_k:=x_k-x_k^f.
	\label{eq:appendix_ek_def}
\end{equation}

\begin{lemma}[One-step tracking recursion]
	\label{lem:one-step-tracking}
	For each step $k$, there exists some $\xi_k$ between $x_k$ and $x_k^f$ such that
	\begin{equation}
		e_{k+1}
		=
		F_k'(\xi_k)e_k
		+
		\left(x_k^f-x_{k+1}^f\right).
		\label{eq:appendix_error_exact}
	\end{equation}
	Consequently,
	\begin{equation}
		|e_{k+1}|
		\le
		|F_k'(\xi_k)|\,|e_k|
		+
		|x_{k+1}^f-x_k^f|.
		\label{eq:appendix_error_bound}
	\end{equation}
\end{lemma}

\begin{proof}
	By definition,
	\begin{equation}
		e_{k+1}
		=
		x_{k+1}-x_{k+1}^f.
	\end{equation}
	Using the update $x_{k+1}=F_k(x_k)$ and adding and subtracting $x_k^f$, we obtain
	\begin{align}
		e_{k+1}
		&=
		F_k(x_k)-x_{k+1}^f \nonumber \\
		&=
		\left(F_k(x_k)-F_k(x_k^f)\right)
		+
		\left(x_k^f-x_{k+1}^f\right),
	\end{align}
	where we used $F_k(x_k^f)=x_k^f$. By the mean value theorem, there exists some $\xi_k$ between $x_k$ and $x_k^f$ such that
	\begin{equation}
		F_k(x_k)-F_k(x_k^f)
		=
		F_k'(\xi_k)(x_k-x_k^f)
		=
		F_k'(\xi_k)e_k.
	\end{equation}
	Substituting this identity gives Eq.~\ref{eq:appendix_error_exact}. Taking absolute values gives Eq.~\ref{eq:appendix_error_bound}.
\end{proof}

\begin{proof}[Proof of Theorem~\ref{thm:bounded_tracking}]
	From Proposition~\ref{prop:local-stability}, the derivative at the instantaneous equilibrium satisfies
	\begin{equation}
		\rho_0
		:=
		|F_k'(x_k^f)|
		=
		|1-\eta z^*|
		<1.
	\end{equation}
	Choose any $\rho$ such that
	\begin{equation}
		\rho_0<\rho<1.
	\end{equation}
	For a displacement $e$ from the instantaneous equilibrium,
	\begin{equation}
		F_k'(x_k^f+e)
		=
		1-\eta e^{x_k^f+e}m_k
		=
		1-\eta z^* e^e.
		\label{eq:derivative_relative_equilibrium}
	\end{equation}
	The right-hand side is independent of $m_k$. Since it is continuous in $e$ and equals $1-\eta z^*$ at $e=0$, there exists a radius $r>0$ such that
	\begin{equation}
		|F_k'(x_k^f+e)|
		\le
		\rho,
		\qquad
		\text{for all } |e|\le r
	\end{equation}
	and for all $k\ge T$.
	
	We now show that this neighborhood is invariant. By assumption,
	\begin{equation}
		|e_T|\le r.
	\end{equation}
	Suppose inductively that $|e_k|\le r$. Since $\xi_k$ lies between $x_k$ and $x_k^f$, it also satisfies
	\begin{equation}
		|\xi_k-x_k^f|
		\le
		|x_k-x_k^f|
		=
		|e_k|
		\le r.
	\end{equation}
	Therefore,
	\begin{equation}
		|F_k'(\xi_k)|
		\le
		\rho.
	\end{equation}
	Using Lemma~\ref{lem:one-step-tracking} and the bounded-drift assumption gives
	\begin{equation}
		|e_{k+1}|
		\le
		\rho |e_k|+\bar{\delta}.
		\label{eq:appendix_basic_recursion}
	\end{equation}
	Since $\bar{\delta}\le(1-\rho)r$,
	\begin{equation}
		|e_{k+1}|
		\le
		\rho r+(1-\rho)r
		=
		r.
	\end{equation}
	Hence, by induction,
	\begin{equation}
		|e_k|\le r,
		\qquad
		\forall k\ge T.
	\end{equation}
	Thus, the local contraction bound used in Eq.~\ref{eq:appendix_basic_recursion} remains valid along the entire tracked trajectory.
	
	Unrolling Eq.~\ref{eq:appendix_basic_recursion} gives
	\begin{equation}
		|e_k|
		\le
		\rho^{k-T}|e_T|
		+
		\bar{\delta}\sum_{j=0}^{k-T-1}\rho^j.
		\label{eq:appendix_geometric_form}
	\end{equation}
	Since $0\le \rho<1$,
	\begin{equation}
		\sum_{j=0}^{k-T-1}\rho^j
		=
		\frac{1-\rho^{k-T}}{1-\rho}
		\le
		\frac{1}{1-\rho}.
	\end{equation}
	Substituting this bound into Eq.~\ref{eq:appendix_geometric_form} yields
	\begin{equation}
		|e_k|
		\le
		\rho^{k-T}|e_T|
		+
		\frac{\bar{\delta}}{1-\rho},
		\qquad
		\forall k\ge T.
	\end{equation}
	Since $e_k=x_k-x_k^f$, this proves the theorem.
\end{proof}

\subsection{Interpretation of the Bound}
\label{sec:tracking_bound_interpretation}
The bound
\begin{equation}
	|e_k|
	\le
	\rho^{k-T}|e_T|
	+
	\frac{\bar{\delta}}{1-\rho}
\end{equation}
contains two terms. The first term, $\rho^{k-T}|e_T|$, is a transient error that decays geometrically. The second term, $\bar{\delta}/(1-\rho)$, is a persistent tracking-error bound controlled by the maximum equilibrium drift. Therefore, the log-space $\beta$ dynamics remains close to the instantaneous equilibrium trajectory as long as the equilibrium drift is bounded and the trajectory stays within the local contraction neighborhood.

\section{A Sufficient Condition for Bounded Equilibrium Drift}
\label{appendix:drift_condition}

The bounded tracking theorem assumes that the moving equilibrium satisfies
\[
|x_{k+1}^f-x_k^f|\le \bar{\delta}.
\]
We show that this condition follows from a simple regularity condition on the margin statistic.

\begin{lemma}[Sufficient condition for bounded equilibrium drift]
	Assume that there exist constants $T\in\mathbb{N}$, $m_{\min}>0$, and $\varepsilon\ge 0$ such that for all $k\ge T$,
	\[
	m_k\ge m_{\min},
	\qquad
	|m_{k+1}-m_k|\le \varepsilon.
	\]
	Then
	\[
	|x_{k+1}^f-x_k^f|
	\le
	\frac{\varepsilon}{m_{\min}},
	\qquad
	\forall k\ge T.
	\]
\end{lemma}

\begin{proof}
	Using $x_k^f=\log(z^*/m_k)=\log z^*-\log m_k$, we have
	\[
	|x_{k+1}^f-x_k^f|
	=
	|\log m_{k+1}-\log m_k|.
	\]
	By the mean value theorem, for some $\zeta_k$ between $m_k$ and $m_{k+1}$,
	\[
	|\log m_{k+1}-\log m_k|
	=
	\frac{1}{\zeta_k}|m_{k+1}-m_k|.
	\]
	Since $m_k,m_{k+1}\ge m_{\min}>0$, we have $\zeta_k\ge m_{\min}$. Therefore,
	\[
	|x_{k+1}^f-x_k^f|
	\le
	\frac{1}{m_{\min}}|m_{k+1}-m_k|
	\le
	\frac{\varepsilon}{m_{\min}}.
	\]
\end{proof}

This lemma provides an explicit bound on the equilibrium drift whenever the margin statistic stays bounded away from zero and changes gradually across training steps. In particular, the drift condition in Theorem~\ref{thm:bounded_tracking} is satisfied whenever $\varepsilon/m_{\min} \le (1-\rho)r$.

\section{Implementation Details}
\label{app:implementation_details}

\subsection{Models and Training Datasets}
\label{app:models_training_datasets}

\paragraph{UltraFeedback experiments.}
For the UltraFeedback experiments, we follow the SimPO~\citep{meng2024simpo} training setup. We first perform supervised fine-tuning of the Meta-Llama-3-8B\footnote{\url{https://huggingface.co/meta-llama/Meta-Llama-3-8B}; Llama 3 Community License.} and Qwen3-8B-Base\footnote{\url{https://huggingface.co/Qwen/Qwen3-8B-Base}; Apache 2.0 License.} models on UltraChat-200k\footnote{\url{https://huggingface.co/datasets/HuggingFaceH4/ultrachat_200k}; MIT License.}~\citep{ding2023enhancing}; the resulting SFT checkpoints are used as the policy initializations and reference models for preference optimization on the UltraFeedback binarized\footnote{\url{https://huggingface.co/datasets/HuggingFaceH4/ultrafeedback_binarized}; MIT License.}~\citep{cui2023ultrafeedback} dataset. We also follow the SimPO protocol for common training settings, including the learning rate, scheduler, and optimizer, and keep them fixed across experiments. Table~\ref{tab:ultrafeedback_lscdpo_hparams} reports these fixed training settings and the LSC-DPO control parameters, with bold entries denoting the main settings used in Table~\ref{tab:main_eval}.
\begin{table}[H]
	\centering
	\small
	\setlength{\tabcolsep}{6.5pt}
	\renewcommand{\arraystretch}{1.0}
	\caption{
		Training configuration of LSC-DPO on UltraFeedback.
	}
	\label{tab:ultrafeedback_lscdpo_hparams}
	
	\begin{tabular}{lcc}
		\toprule
		\textbf{Configuration} & \textbf{Llama3-8B} & \textbf{Qwen3-8B} \\
		\midrule
		Model
		& meta-llama/Meta-Llama-3-8B
		& Qwen/Qwen3-8B-Base \\
		Optimizer
		& AdamW
		& AdamW \\
		Epoch
		& 1
		& 1 \\
		Batch size
		& 128
		& 128 \\
		Learning rate
		& 5e-7
		& 5e-7 \\
		Scheduler
		& cosine
		& cosine \\
		Warmup ratio
		& 0.1
		& 0.1 \\
		Weight decay
		& 0
		& 0 \\
		Precision
		& BF16
		& BF16 \\
		$\beta_0$
		& 0.01
		& 0.01 \\
		$\eta$
		& 0.1
		& 0.1 \\
		
		$z^*$
		& \begin{tabular}{c}
			$\{0.20,0.30,0.35,\mathbf{0.40},$ \\
			$0.45,0.50,0.60,0.70\}$
		\end{tabular}
		& \begin{tabular}{c}
			$\{0.20,0.30,\mathbf{0.35},0.40,$ \\
			$0.45,0.50,0.60,0.70\}$
		\end{tabular} \\
		
		$\tau_{\mathrm{start}}$
		& $\{0.40,0.43,\mathbf{0.45}\}$
		& $\{0.40,0.43,\mathbf{0.45}\}$ \\
		Device
		& 4$\times$H200 GPUs
		& 4$\times$H200 GPUs \\
		\bottomrule
	\end{tabular}
\end{table}

\paragraph{Anthropic-HH experiments.}
For the Anthropic-HH~\citep{bai2022training}\footnote{\url{https://huggingface.co/datasets/Anthropic/hh-rlhf}; MIT License.} experiments, we use both the harmless-base and helpful-base subsets. We follow DPO~\citep{rafailov2023direct} and $\beta$-DPO~\citep{wu2024beta} for the batch size (64) and single epoch, $\varepsilon$-DPO~\citep{lee2025kl} for the optimizer, scheduler and warmup, and keep the learning rate of $5\times10^{-7}$ shared with our UltraFeedback runs. For method-specific hyperparameters of $\beta$-DPO and $\varepsilon$-DPO, we use the recommended values reported in the respective papers. Tables~\ref{tab:hh_sft_config} and~\ref{tab:hh_pref_config} summarize the SFT-stage and preference-optimization-stage configurations, respectively.
\begin{table}[H]
	\centering
	\small
	\setlength{\tabcolsep}{8pt}
	\caption{
		SFT training configuration for Anthropic-HH.
	}
	\label{tab:hh_sft_config}
	\begin{tabular}{lcc}
		\toprule
		\textbf{Configuration} & \textbf{Llama3-8B} & \textbf{Qwen3-8B} \\
		\midrule
		Model & meta-llama/Meta-Llama-3-8B & Qwen/Qwen3-8B-Base \\
		Epochs & 1 & 1 \\
		Batch size & 64 & 64 \\
		Scheduler & cosine & cosine \\
		Warmup ratio & 0.1 & 0.1 \\
		Weight decay & 0 & 0 \\
		Learning rate & 2e-5 & 2e-5 \\
		Precision & BF16 & BF16 \\
		Device & 4$\times$H200 GPUs & 4$\times$H200 GPUs \\
		\bottomrule
	\end{tabular}
\end{table}
\begin{table}[H]
	\centering
	\small
	\setlength{\tabcolsep}{8pt}
	\caption{
		Training configurations of LSC-DPO on Anthropic-HH.
	}
	\label{tab:hh_pref_config}
	\begin{tabular}{lcc}
		\toprule
		\textbf{Configuration} & \textbf{Llama3-8B} & \textbf{Qwen3-8B} \\
		\midrule
		Model & meta-llama/Meta-Llama-3-8B & Qwen/Qwen3-8B-Base \\
		Epochs & 1 & 1 \\
		Batch size & 64 & 64 \\
		Scheduler & cosine & cosine \\
		Warmup ratio & 0.1 & 0.1 \\
		Weight decay & 0 & 0 \\
		Learning rate & 5e-7 & 5e-7 \\
		Precision & BF16 & BF16 \\
		$\beta_0$ & 0.1 & 0.1 \\
		$\tau_{\mathrm{start}}$ 
		& $\{0.40,0.43,\textbf{0.45},0.48,0.50\}$ 
		& $\{0.40,0.43,\textbf{0.45},0.48,0.50\}$ \\
		$z^*$ 
		& $\{0.35,\textbf{0.40},0.45,0.60,0.80,0.85\}$ 
		& $\{0.35,\textbf{0.40},0.45,0.60,0.80,0.85\}$ \\
		$\eta$ & \{0.01, 0.05, 0.1, 0.3, 0.5, 1, 5, 8\} & \{0.01, 0.05, 0.1, 0.3, 0.5, 1, 5, 8\} \\
		Device & 4$\times$H200 GPUs & 4$\times$H200 GPUs \\
		\bottomrule
	\end{tabular}
\end{table}

\paragraph{Computation environment.}
All training experiments in this paper are based on the Hugging Face alignment-handbook\footnote{\url{https://github.com/huggingface/alignment-handbook}; Apache 2.0 License.} and TRL\footnote{\url{https://github.com/huggingface/trl}; Apache 2.0 License.}. In our setup with 4 H200 GPUs, each LSC-DPO training run on UltraFeedback completes within two hours, while each run on the Anthropic-HH helpful-base or harmless-base subset completes within one hour.

\subsection{Evaluation Benchmarks and Baselines}
\label{app:evaluation_benchmarks_baselines}

\paragraph{UltraFeedback evaluation.}
For the UltraFeedback experiments, we follow the SimPO open-alignment evaluation protocol~\citep{meng2024simpo} and the official benchmark settings of AlpacaEval~2~\citep{dubois2024length} and MT-Bench~\citep{zheng2023judging}. 
We compare LSC-DPO with SFT, DPO~\citep{rafailov2023direct}, $\beta$-DPO~\citep{wu2024beta}, $\varepsilon$-DPO~\citep{lee2025kl}, R-DPO~\citep{park2024disentangling}, IPO~\citep{azar2024general}, KTO~\citep{ethayarajh2024kto}, CPO~\citep{xu2024contrastive}, ORPO~\citep{hong2024orpo}, SLiC-HF~\citep{zhao2023slic}, and SimPO~\citep{meng2024simpo}.
All methods are evaluated with the same generation settings and evaluation pipeline to ensure a fair comparison. 
On AlpacaEval~2, we report both the raw win rate and the length-controlled win rate, with the length-controlled metric used as the primary metric because preference optimization can increase response length. 
On MT-Bench, we report the average score over both turns. 

\paragraph{Anthropic-HH evaluation.}
For the Anthropic-HH experiments, we follow the evaluation protocol used in DPO~\citep{rafailov2023direct} and $\beta$-DPO~\citep{wu2024beta}. 
We evaluate harmless-base and helpful-base separately by comparing each model response with the HH chosen response, using \texttt{gpt-4-0613} as the judge with randomized response order. We report the win rate of the model response. 
All methods use the same decoding configuration, with temperature $0.7$ and top-$p$ sampling ($p=0.9$), to ensure a fair comparison.

\paragraph{Control-parameter selection.}
We selected the LSC-DPO control parameters using a separate Anthropic-HH development subset before the main AlpacaEval~2 evaluation. This preliminary study identified a coarse candidate region around $z^*\in[0.35,0.45]$ and $\tau_{\mathrm{start}}\in[0.40,0.45]$. The AlpacaEval~2 results in Appendix~\ref{sec:zstar_sensitivity} are therefore reported as a target-sensitivity analysis rather than as an unrestricted benchmark-specific hyperparameter search. After the target and training settings were fixed, the same trained checkpoints were evaluated on MT-Bench without additional tuning.

\paragraph{Arena-Hard reproducibility considerations.}
We also considered Arena-Hard~\citep{li2024crowdsourced} as an additional open-ended benchmark, but do not include it in the main evaluation because prior Arena-Hard v0.1 results for SimPO~\citep{meng2024simpo} and $\varepsilon$-DPO~\citep{lee2025kl} used \texttt{gpt-4-1106-preview} as the judge. OpenAI has since retired this model, so the original protocol cannot be reproduced in our evaluation environment. To examine how the released checkpoints behave under currently available judging setups, we re-evaluated SimPO's released Llama3-8B-base preference-optimization checkpoints under two settings: Arena-Hard v0.1 judged with \texttt{gpt-4.1}, and Arena-Hard v2.0 hard-prompt judged with \texttt{gpt-4.1}. Table~\ref{tab:arena_hard_reeval} shows that all evaluated preference-optimization methods score below the SFT baseline under these updated settings, in contrast to their behavior on AlpacaEval~2 and MT-Bench. The relative ordering also differs substantially from prior Arena-Hard reports, suggesting that results obtained under the updated judging setups may not be directly comparable to those reported under the original protocol. We therefore use AlpacaEval~2 and MT-Bench as the main evaluations.
\begin{table}[H]
	\centering
	\footnotesize
	\setlength{\tabcolsep}{4pt}
	\caption{
		Re-evaluation of the released Llama3-8B-base preference-optimization
		checkpoints from the SimPO study on Arena-Hard.
		Panel (a) reports Arena-Hard v0.1 results using \texttt{gpt-4.1};
		Panel (b) reports Arena-Hard v2.0 hard-prompt results using
		\texttt{gpt-4.1}. The original v0.1 judge,
		\texttt{gpt-4-1106-preview}, is no longer available in our evaluation
		environment.
	}
	\label{tab:arena_hard_reeval}
	
	\begin{minipage}{0.485\textwidth}
		\centering
		\begin{tabular}{lcc}
			\toprule
			Model & Score (\%) & 95\% CI (\%) \\
			\midrule
			\texttt{gpt-4-0314} (ref.) & 50.0 & $(-0.0/+0.0)$ \\
			SFT           & 26.4 & $(-1.9/+1.7)$ \\
			SFT + CPO     & 22.4 & $(-1.3/+1.6)$ \\
			SFT + SimPO   & 20.1 & $(-1.7/+2.1)$ \\
			SFT + DPO     & 16.3 & $(-1.4/+1.6)$ \\
			SFT + R-DPO   & 14.5 & $(-1.3/+1.3)$ \\
			SFT + IPO     & 13.4 & $(-1.3/+1.5)$ \\
			SFT + ORPO    & 12.4 & $(-1.2/+1.2)$ \\
			SFT + KTO     & 10.3 & $(-1.1/+1.1)$ \\
			SFT + SLiC-HF & 7.7  & $(-0.9/+0.8)$ \\
			\bottomrule
		\end{tabular}
		
		\vspace{0.35em}
		{\small (a) Arena-Hard v0.1, judge: \texttt{gpt-4.1}.}
	\end{minipage}
	\hfill
	\begin{minipage}{0.485\textwidth}
		\centering
		\begin{tabular}{lcc}
			\toprule
			Model & Score (\%) & 95\% CI (\%) \\
			\midrule
			\texttt{o3-mini} (ref.) & 50.0 & $(-0.0/+0.0)$ \\
			SFT           & 9.2  & $(-0.8/+0.8)$ \\
			SFT + CPO     & 7.4  & $(-0.7/+0.9)$ \\
			SFT + DPO     & 2.3  & $(-0.4/+0.4)$ \\
			SFT + SimPO   & 1.8  & $(-0.4/+0.4)$ \\
			SFT + IPO     & 1.4  & $(-0.3/+0.4)$ \\
			SFT + R-DPO   & 1.4  & $(-0.3/+0.3)$ \\
			SFT + ORPO    & 1.1  & $(-0.4/+0.4)$ \\
			SFT + KTO     & 0.7  & $(-0.2/+0.3)$ \\
			SFT + SLiC-HF & 0.5  & $(-0.2/+0.2)$ \\
			\bottomrule
		\end{tabular}
		
		\vspace{0.35em}
		{\small (b) Arena-Hard v2.0 hard-prompt, judge: \texttt{gpt-4.1}.}
	\end{minipage}
\end{table}

\section{Additional Learning-Signal Trajectories on Anthropic-HH}
\label{app:additional_dynamics}
This appendix provides additional learning-signal trajectories for the Anthropic-HH analysis in Section~\ref{learning_signal_analysis}. We first report the raw per-step trajectories corresponding to the EMA-smoothed Llama3-8B results, and then provide the corresponding Qwen3-8B results as a cross-model comparison.

\subsection{Raw Learning-Signal Trajectories on Llama3-8B}
Figures~\ref{fig:harmless_llama_raw} and~\ref{fig:helpful_llama_raw} show the raw per-step trajectories together with their EMA-smoothed trends for Llama3-8B. Compared with the smoothed curves reported in Section~\ref{learning_signal_analysis}, the raw trajectories make the short-term fluctuations in the learning signal more visible, especially for $\beta$-DPO. This behavior suggests that dynamically adjusting $\beta$ based on margin statistics can introduce large short-term fluctuations in the learning signal. In contrast, LSC-DPO maintains a more stable smoothed trajectory while avoiding the persistent low-signal regime observed under DPO, particularly on the helpful-base subset.
\begin{figure}[H]
	\centering
	\begin{subfigure}{0.48\textwidth}
		\centering
		\includegraphics[width=\linewidth]{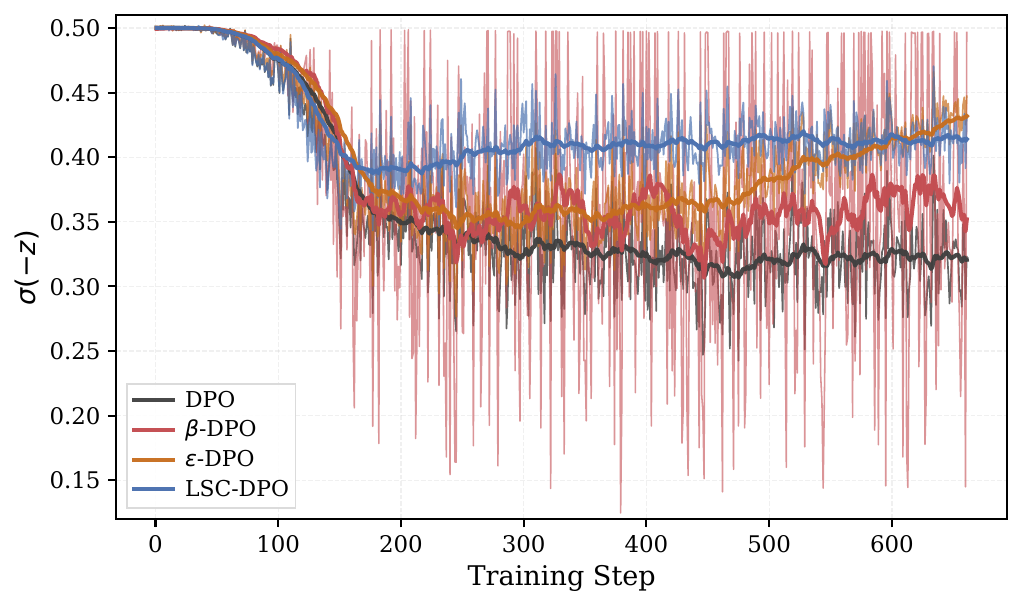}
		\caption{harmless-base}
		\label{fig:harmless_llama_raw}
	\end{subfigure}
	\hfill
	\begin{subfigure}{0.48\textwidth}
		\centering
		\includegraphics[width=\linewidth]{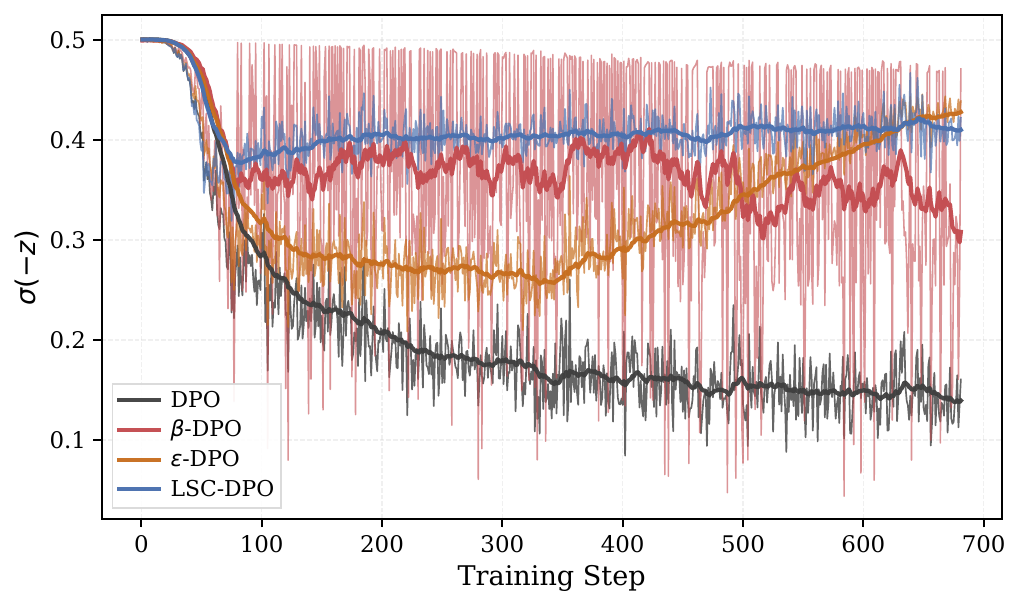}
		\caption{helpful-base}
		\label{fig:helpful_llama_raw}
	\end{subfigure}
	
	\caption{
		Raw and EMA-smoothed trajectories of the learning signal $\sigma(-z)$ on Anthropic-HH with Llama3-8B. Faint curves show raw per-step values, while bold curves show EMA-smoothed trends.
	}
	\label{fig:app_llama_qt_raw}
\end{figure}

\subsection{Learning-Signal Trajectories on Qwen3-8B}
\label{app:raw_qt}
Figures~\ref{fig:app_qwen_qt_raw} and~\ref{fig:app_qwen_qt_ema} report the corresponding learning-signal trajectories for Qwen3-8B. The same qualitative pattern observed with Llama3-8B also appears for Qwen3-8B. DPO drives the learning signal toward a lower regime as training proceeds, while $\beta$-DPO exhibits larger fluctuations. $\varepsilon$-DPO partially recovers the signal in later stages. Across both harmless-base and helpful-base, LSC-DPO maintains a higher and more stable learning signal, providing additional evidence that the proposed control rule is not specific to a single model family.
\begin{figure}[H]
	\centering
	\begin{subfigure}{0.48\textwidth}
		\centering
		\includegraphics[width=\linewidth]{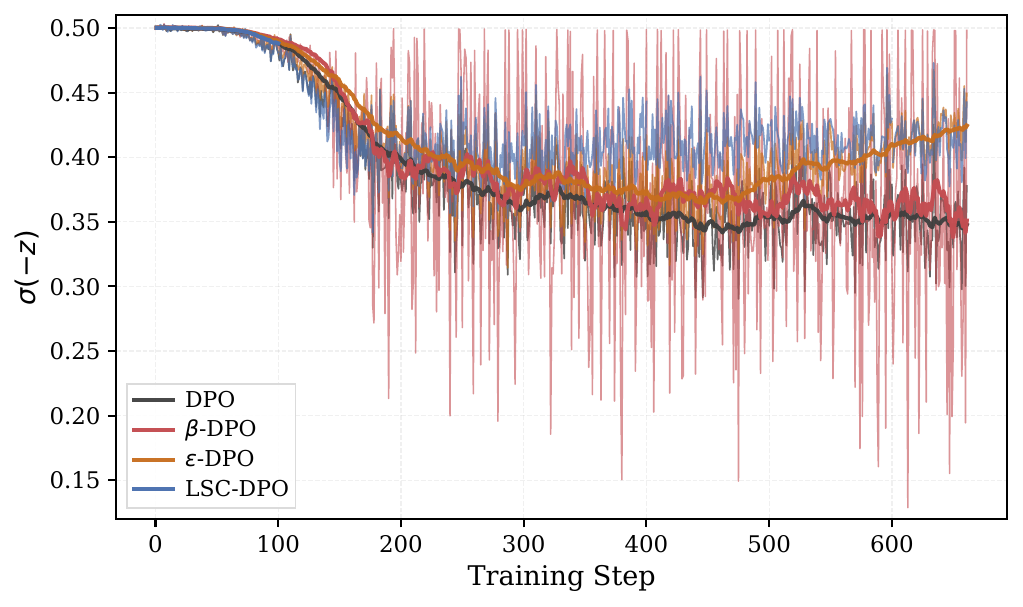}
		\caption{harmless-base}
		\label{fig:qwen_raw_harmless}
	\end{subfigure}
	\hfill
	\begin{subfigure}{0.48\textwidth}
		\centering
		\includegraphics[width=\linewidth]{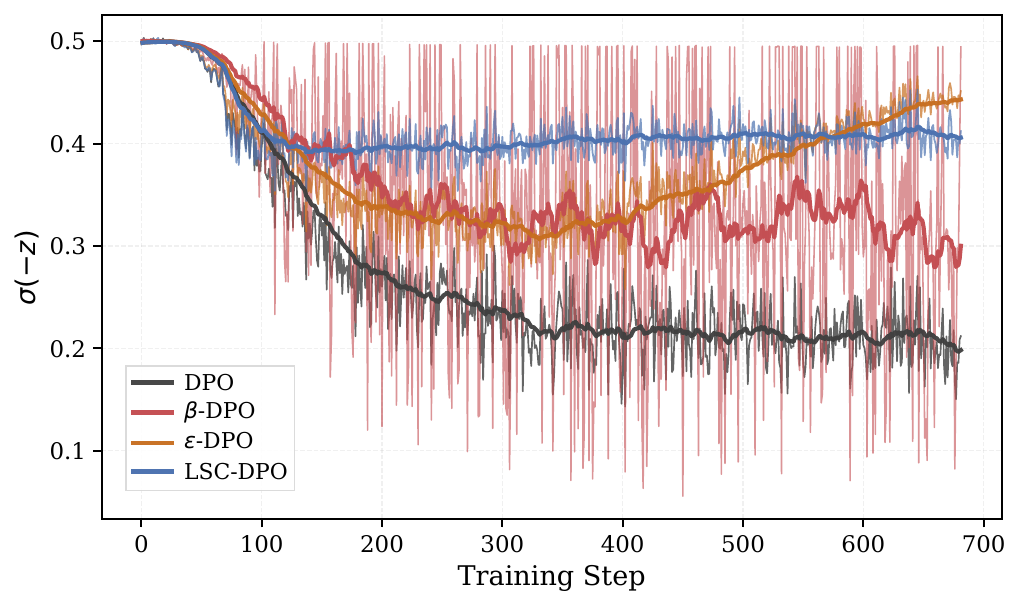}
		\caption{helpful-base}
		\label{fig:qwen_raw_helpful}
	\end{subfigure}
	
	\caption{
		Raw and EMA-smoothed trajectories of the learning signal $\sigma(-z)$
		on Anthropic-HH with Qwen3-8B.
		Faint curves denote raw per-step values, while bold curves denote EMA-smoothed trends.
	}
	\label{fig:app_qwen_qt_raw}
\end{figure}
\begin{figure}[H]
	\centering
	\begin{subfigure}{0.48\textwidth}
		\centering
		\includegraphics[width=\linewidth]{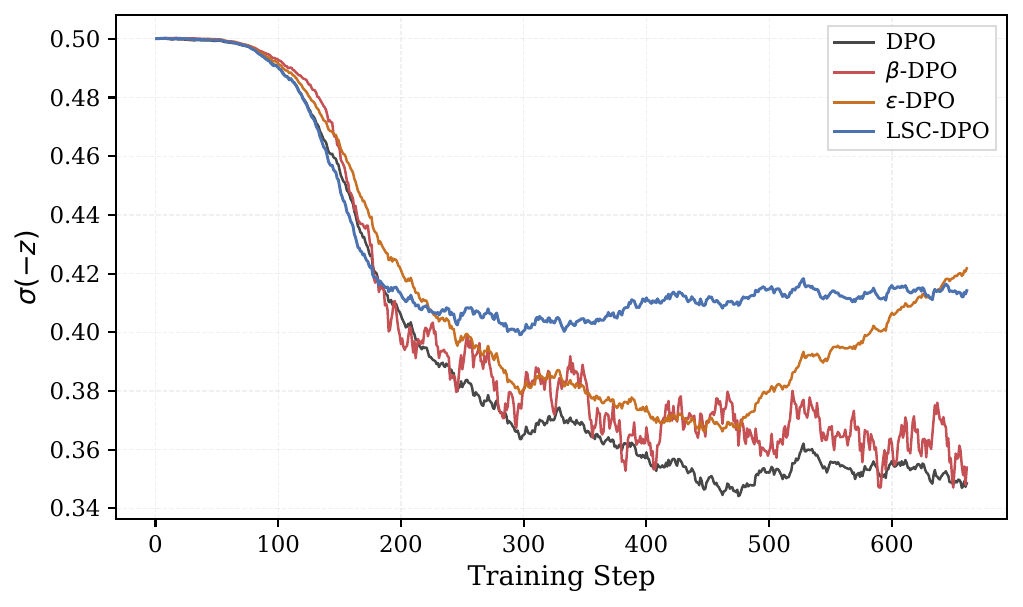}
		\caption{harmless-base}
		\label{fig:qwen_ema_harmless}
	\end{subfigure}
	\hfill
	\begin{subfigure}{0.48\textwidth}
		\centering
		\includegraphics[width=\linewidth]{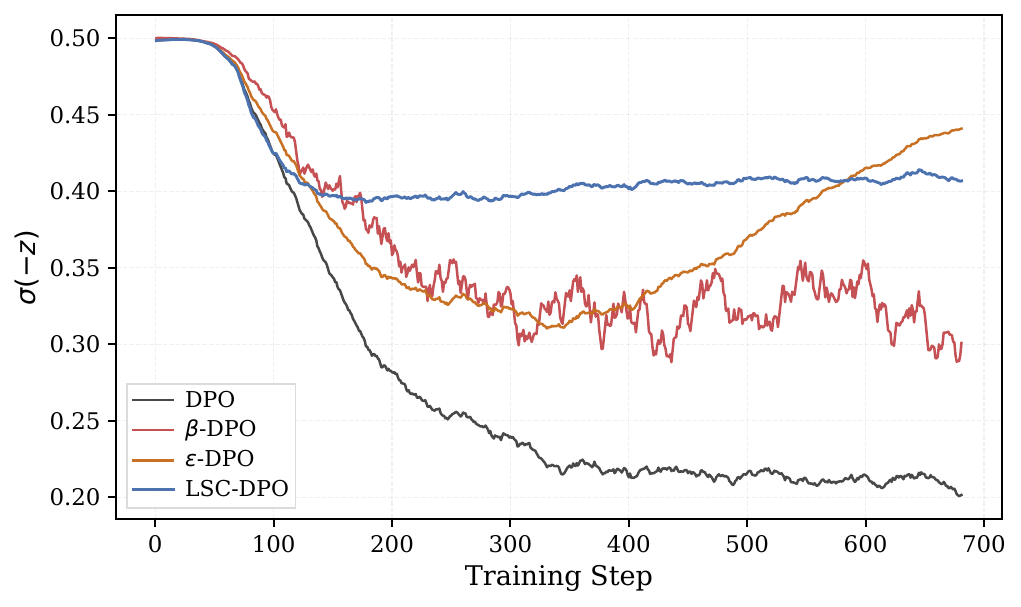}
		\caption{helpful-base}
		\label{fig:qwen_ema_helpful}
	\end{subfigure}
	
	\caption{
		EMA-smoothed trajectories of the learning signal $\sigma(-z)$ on Anthropic-HH with Qwen3-8B.
	}
	\label{fig:app_qwen_qt_ema}
\end{figure}

\section{Stability under Different Update Rates}
\label{stability_analysis}
We further examine the stability behavior of LSC-DPO by varying the control rate $\eta$ while fixing the target scaled margin to $z^*=0.4$. Under this setting, the target learning signal is $\sigma(-z^*)\approx 0.4013$, and changing $\eta$ directly changes the stability-relevant product $\eta z^*$ in Proposition~\ref{prop:local-stability}. Figures~\ref{fig:eta_qt_stability_llama} and \ref{fig:eta_qt_stability_qwen} show the EMA-smoothed trajectories of the learning signal $\sigma(-z)$ on Anthropic-HH subsets. Very small values of $\eta$ react slowly to deviations from the target signal. For example, when $\eta=0.01$, the signal often drops below the target level and recovers only gradually. In contrast, moderate values such as $\eta=0.3$, $0.5$, and $1.0$ provide more responsive correction and keep the signal close to the target after the initial adjustment phase across both model families and subsets.

With $z^*=0.4$, the local stability condition $0<\eta z^*<2$ reduces to $\eta<5$. At $\eta=5$, which lies exactly on the boundary with $|1-\eta z^*|=1$, the local contraction is lost and the learning-signal trajectory exhibits substantially larger fluctuations. In contrast, $\eta=8$ violates the stability condition and fails to
reliably track the target signal. These results support the feedback-control
interpretation of LSC-DPO: appropriately bounded control rates maintain the learning signal near the target regime, whereas overly large rates can destabilize the signal dynamics.
\begin{figure}[H]
	\centering
	
	\begin{subfigure}{0.48\linewidth}
		\centering
		\includegraphics[width=\linewidth]{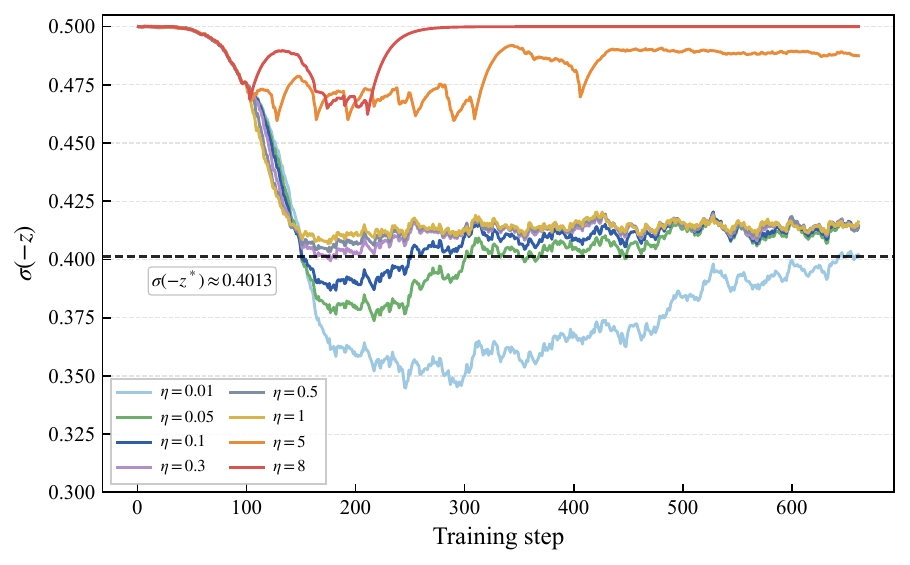}
		\caption{harmless-base}
	\end{subfigure}
	\hfill
	\begin{subfigure}{0.48\linewidth}
		\centering
		\includegraphics[width=\linewidth]{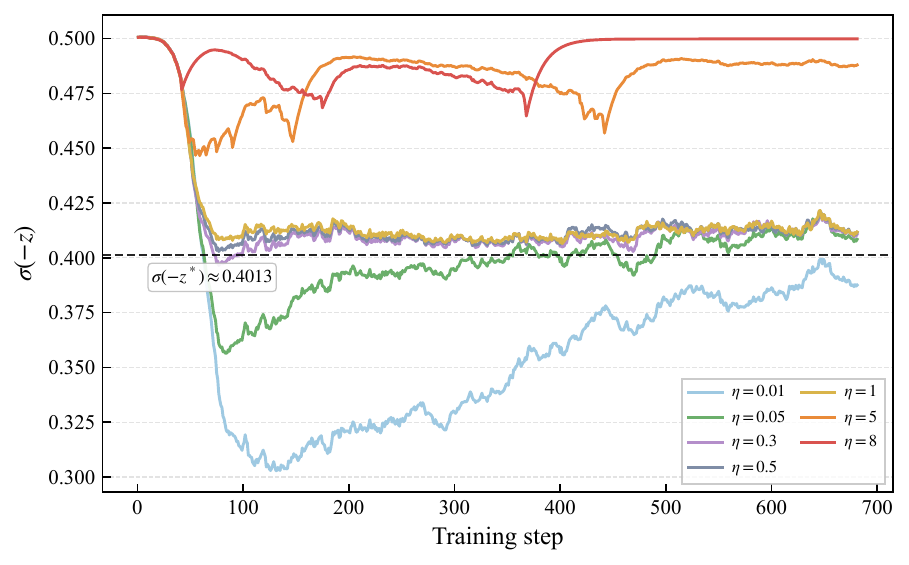}
		\caption{helpful-base}
	\end{subfigure}
	
	\caption{
		Learning signal $\sigma(-z)$ trajectories of LSC-DPO under different
		control rates $\eta$ with Llama3-8B on Anthropic-HH subsets.
		The dashed line marks the target signal
		$\sigma(-z^*)\approx 0.4013$.
	}
	\label{fig:eta_qt_stability_llama}
	
	\setcounter{subfigure}{0}
	
	\begin{subfigure}{0.48\linewidth}
		\centering
		\includegraphics[width=\linewidth]{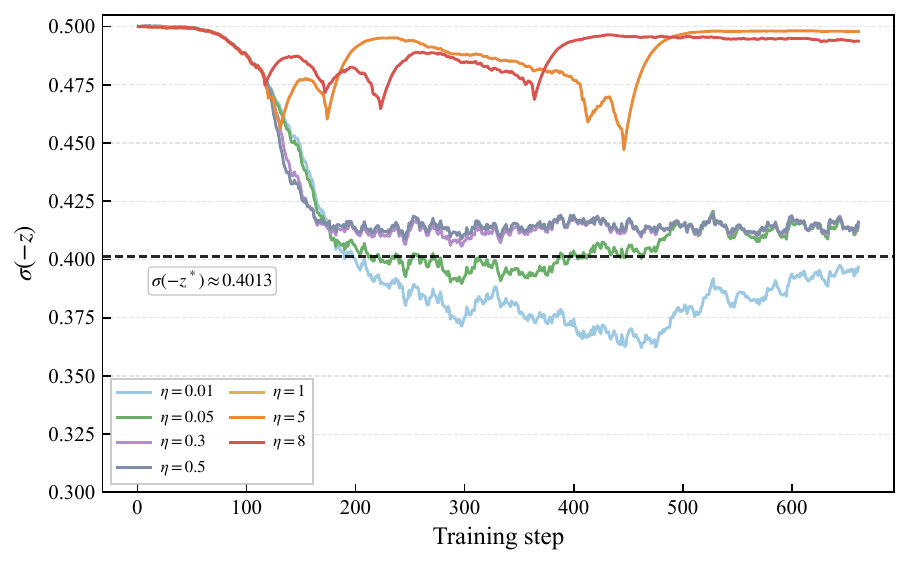}
		\caption{harmless-base}
	\end{subfigure}
	\hfill
	\begin{subfigure}{0.48\linewidth}
		\centering
		\includegraphics[width=\linewidth]{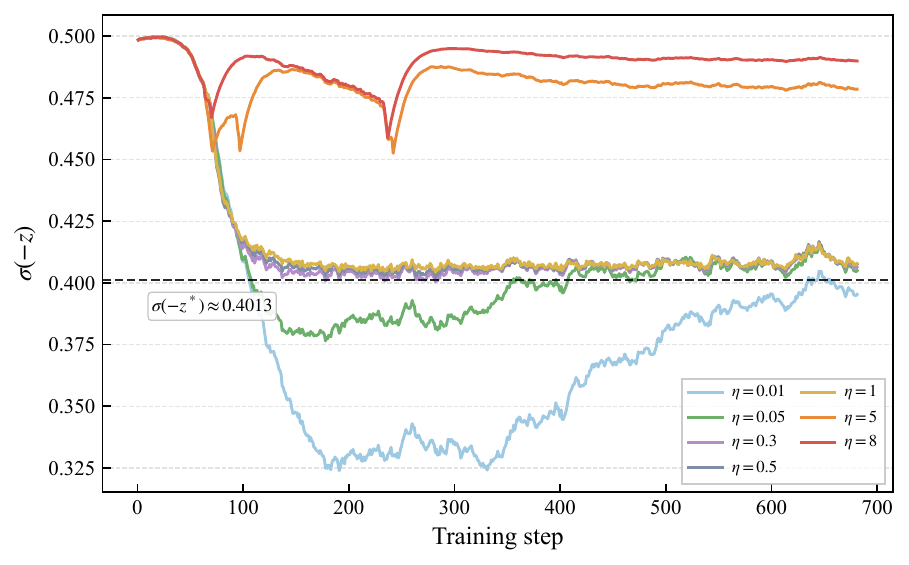}
		\caption{helpful-base}
	\end{subfigure}
	
	\caption{
		Learning signal $\sigma(-z)$ trajectories of LSC-DPO under different
		control rates $\eta$ with Qwen3-8B on Anthropic-HH subsets.
		The dashed line marks the target signal
		$\sigma(-z^*)\approx 0.4013$.
	}
	\label{fig:eta_qt_stability_qwen}
	
\end{figure}

\section{Additional Mechanistic Validation of Learning-Signal Control}
\label{app:optimization_dynamics}
The mechanistic experiments in Section~\ref{sec:mechanistic_validation}
use the same Llama3-8B SFT checkpoint and UltraFeedback training configuration
as the main experiments. Unless otherwise specified, all training and evaluation settings are kept fixed across the compared methods.

\subsection{Learning Signal and Realized Optimization Dynamics}
\label{app:optimization_dynamics_loss}
We compute the curvature factor as $c_t=\sigma(-z_t)(1-\sigma(-z_t))$ and report the gradient norm measured before gradient clipping. LSC-DPO control is activated at step 84. For visualization, solid curves show EMA-smoothed trajectories, whereas all
phase-averaged statistics in Table~\ref{tab:mechanistic_diagnostics} are computed from the original unsmoothed measurements.
\begin{figure}[H]
	\centering
	\includegraphics[width=0.8\textwidth]{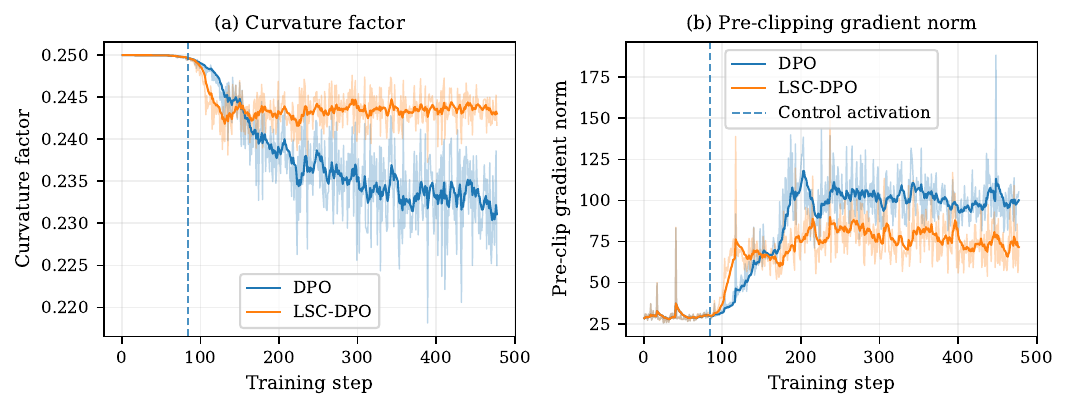}
	\caption{
		Optimization diagnostics for DPO and LSC-DPO throughout training.
		Left: curvature factor. Right: pre-clipping gradient norm.
		The vertical dashed line marks LSC-DPO control activation at step 84.
	}
	\label{fig:optimization_dynamics}
\end{figure}
Before control activation, DPO and LSC-DPO follow nearly identical trajectories.
After activation, LSC-DPO consistently maintains a larger curvature factor,
whereas DPO continues moving toward a lower-curvature regime. The pre-clipping
gradient trajectories also separate after control activation. These full
trajectories confirm that the phase-averaged differences reported in
Table~\ref{tab:mechanistic_diagnostics} persist throughout the controlled stage
rather than arising from a small number of training steps.

\subsection{Feedback versus Predetermined $\beta$ Schedules}
\label{app:beta_schedule_comparison}
For the predetermined-schedule baselines in Table~\ref{tab:beta_schedule_comparison}, we vary $\beta$ according to linear, cosine, and exponential schedules fixed before training. Let $t=0,\ldots,T-1$ index the $\beta$ updates and $f_t=t/(T-1)$. All schedules start from $\beta_0$ and end at $\beta_{\mathrm{end}}$, and are defined as
\begin{align}
	\beta_t^{\mathrm{lin}}
	&=
	\beta_0 + f_t(\beta_{\mathrm{end}}-\beta_0), \\
	\beta_t^{\mathrm{cos}}
	&=
	\beta_0
	+\frac{1}{2}\left(1-\cos(\pi f_t)\right)
	(\beta_{\mathrm{end}}-\beta_0), \\
	\beta_t^{\mathrm{exp}}
	&=
	\exp\left(
	\log\beta_0
	+
	f_t(\log\beta_{\mathrm{end}}-\log\beta_0)
	\right).
\end{align}
All predetermined schedules use the same initial and final coefficient values as LSC-DPO, differ only in how $\beta$ evolves between these endpoints, and update the coefficient once per gathered microbatch.

\section{Additional Results on Initialization Robustness and Signal Compensation}
\label{app:signal_compensation}
The experiments in Section~\ref{sec:init_beta} use the same Llama3-8B, UltraFeedback, training, and evaluation settings as those in Section~\ref{sec:mechanistic_validation}; only $\beta_0$ and, where applicable, the compensation target are varied.
\subsection{Initialization Sensitivity: Full LC/WR Results}
\label{app:init_full_results}
Table~\ref{tab:init_beta_full} reports the complete AlpacaEval~2 results for
the initialization-sensitivity study, including both length-controlled (LC)
win rate and standard win rate (WR). The $\beta_0=0.01$ configuration
corresponds to the reference setting used in the main experiments.
\begin{table}[H]
	\caption{
		Full initialization-sensitivity results on AlpacaEval 2.
		We report length-controlled win rate (LC) and standard win rate (WR),
		both in percentage points.
		$\dagger$ denotes the reference setting.
	}
	\label{tab:init_beta_full}
	\centering
	\scriptsize
	\setlength{\tabcolsep}{2.7pt}
	\renewcommand{\arraystretch}{0.95}
	\setlength{\heavyrulewidth}{0.9pt}
	
	\begin{tabular*}{\textwidth}
		{@{\extracolsep{\fill}}lcccccccccc@{}}
		\toprule
		&
		\multicolumn{2}{c}{$\beta_0=0.01^\dagger$} &
		\multicolumn{2}{c}{$\beta_0=0.05$} &
		\multicolumn{2}{c}{$\beta_0=0.10$} &
		\multicolumn{2}{c}{$\beta_0=0.30$} &
		\multicolumn{2}{c}{$\beta_0=0.80$} \\
		\cmidrule(lr){2-3}
		\cmidrule(lr){4-5}
		\cmidrule(lr){6-7}
		\cmidrule(lr){8-9}
		\cmidrule(lr){10-11}
		
		Method
		& LC & WR
		& LC & WR
		& LC & WR
		& LC & WR
		& LC & WR \\
		\midrule
		
		DPO
		& 19.2 & 17.7
		& 10.4 & 7.7
		& 8.7  & 5.8
		& 8.3  & 4.8
		& 6.6  & 3.9 \\
		
		$\beta$-DPO
		& 16.6 & 16.5
		& 10.5 & 7.8
		& 10.8 & 7.8
		& 7.3  & 4.7
		& 6.9  & 3.8 \\
		
		$\varepsilon$-DPO
		& 19.6 & 17.6
		& 15.6 & 13.2
		& 12.9 & 9.8
		& 8.8  & 5.9
		& 6.9  & 4.5 \\
		
		\textbf{LSC-DPO}
		& \textbf{22.7} & \textbf{21.1}
		& \textbf{19.6} & \textbf{18.3}
		& \textbf{19.1} & \textbf{17.2}
		& \textbf{15.3} & \textbf{14.1}
		& \textbf{15.2} & \textbf{13.4} \\
		
		\bottomrule
	\end{tabular*}
\end{table}

\subsection{Additional Signal-Budget Compensation Results}
\label{app:signal_compensation_details}
We provide additional results for the signal-budget compensation experiments
described in Section~\ref{sec:init_beta}. The $\beta_0=0.01$ run is used as
the anchor configuration, and the baseline LSC-DPO target is $z^*=0.400$.
For each non-anchor initialization, the full-compensation target is determined
by the signal-budget rule in Eq.~\ref{eq:signal_compensation}. For
$\beta_0=0.10$ and $\beta_0=0.80$, we additionally evaluate a half-compensation
setting, obtained by applying half of the target shift from the baseline
$z^*=0.400$ toward the corresponding full-compensation target. All other
training and evaluation settings are kept unchanged.
\begin{table}[H]
	\caption{
		Detailed signal-budget compensation results on AlpacaEval~2.
		LC and WR denote length-controlled and standard win rates, respectively.
	}
	\label{tab:compensation_full_results}
	\centering
	\small
	\setlength{\heavyrulewidth}{0.9pt}
	\renewcommand{\arraystretch}{0.95}
	
	\begin{tabular*}{0.82\textwidth}
		{@{\extracolsep{\fill}}lccc@{}}
		\toprule
		Setting & $z^*$ & LC (\%) & WR (\%) \\
		\midrule
		$\beta_0=0.05$            & 0.359 & 21.5 & 20.1 \\
		$\beta_0=0.10$ (half)     & 0.374 & 20.7 & 18.9 \\
		$\beta_0=0.10$ (full)     & 0.348 & 21.1 & 19.5 \\
		$\beta_0=0.30$            & 0.333 & 20.4 & 19.4 \\
		$\beta_0=0.80$ (half)     & 0.362 & 17.4 & 15.8 \\
		$\beta_0=0.80$ (full)     & 0.324 & 19.7 & 18.3 \\
		\bottomrule
	\end{tabular*}
\end{table}
Compared with the corresponding uncompensated runs reported in
Table~\ref{tab:init_beta_full}, full compensation improves both LC and WR
across all tested non-anchor initializations. For the
two initializations where both half and full compensation are evaluated,
half compensation yields intermediate performance, while full compensation
produces larger improvements. This pattern is consistent with the intended
effect of progressively correcting the initialization-induced learning-signal
deficit.

To directly examine whether the target adjustment produces the intended
signal-level correction, we compare the cumulative learning-signal recovery
under half and full compensation. The required recovery is defined relative
to the anchor run and represents the cumulative signal deficit that must be
recovered to match its signal budget.
\begin{figure}[H]
	\centering
	\includegraphics[width=0.95\linewidth]{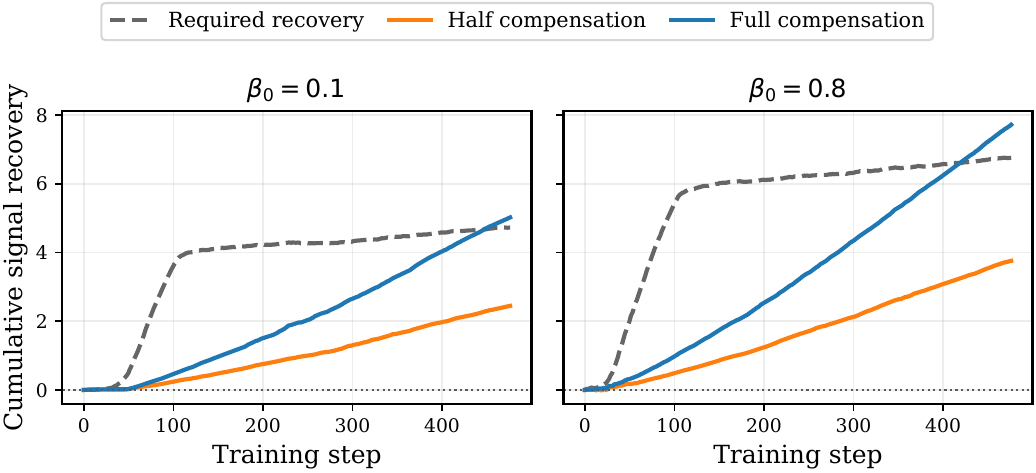}
	\caption{
		Cumulative learning-signal recovery under half and full signal-budget
		compensation for $\beta_0=0.10$ and $\beta_0=0.80$. The dashed curves denote
		the recovery required to match the anchor run ($\beta_0=0.01$), while the
		solid curves show the recovery achieved by half and full compensation.
		Full compensation closely matches or slightly exceeds the required recovery,
		whereas half compensation recovers approximately half of the deficit.
	}
	\label{fig:signal_recovery}
\end{figure}
Figure~\ref{fig:signal_recovery} shows that the realized cumulative recovery
follows the intended compensation strength. Full compensation approaches the
required cumulative recovery by the end of training, whereas half compensation
produces a substantially smaller correction. These results verify that the
target adjustment acts on the cumulative learning signal in the intended
direction. Matching the cumulative signal budget does not imply identical
optimization trajectories or policies, but provides a direct mechanism-level
check of the proposed compensation rule.

\section{Ablation Analysis of Parameters}
\label{ablation_analysis}
\subsection{Target Learning Signal}
\label{sec:zstar_sensitivity}
Table~\ref{tab:zstar_ablation} analyzes the target learning signal in LSC-DPO
on AlpacaEval~2. In our update rule, the target scaled margin $z^*$ specifies
the target signal level $\sigma(-z^*)$ and therefore determines the learning-signal regime that the method attempts to track. We also include DPO as a reference point; its reported signal is the final stabilized value observed during training rather than a controlled target.

The results indicate that stronger performance is achieved in an intermediate
target-signal region rather than at either extreme. When $z^*$ is too small,
the target signal remains high and the loss stays in a relatively high-sensitivity, high-curvature region. The empirical results show that maintaining an overly high target signal does not yield the best performance. Conversely, when $z^*$ is too large, the target signal becomes small and the sigmoid loss moves toward a flatter, lower-sensitivity region. For Llama3-8B, increasing $z^*$ from $0.40$ to $0.70$ reduces the LC win rate from $22.7$ to $19.5$; for Qwen3-8B, increasing $z^*$ from $0.35$ to $0.70$ reduces the LC win rate from $22.1$ to $13.7$.

The best results are obtained at moderate target signal levels. Llama3-8B performs best at $z^*=0.40$, corresponding to $\sigma(-z^*)=0.4013$, while Qwen3-8B performs best at $z^*=0.35$, corresponding to $\sigma(-z^*)=0.4134$. Compared with DPO, the best LSC-DPO settings improve the LC win rate from $19.2$ to $22.7$ on Llama3-8B and from $17.2$ to $22.1$ on Qwen3-8B. Overall, these results suggest that LSC-DPO benefits from an intermediate target learning signal: in our experiments, the strongest settings concentrate around $\sigma(-z^*)\approx 0.40$--$0.41$, where the loss retains meaningful local responsiveness without staying too close to the overly high-signal region or drifting into saturation.
\begin{table}[H]
	\centering
	\small
	\setlength{\tabcolsep}{5pt}
	\caption{
		Target learning signal on AlpacaEval~2.
		The DPO row reports the final stabilized signal under fixed-$\beta$ training.
		Scores are official AlpacaEval~2 point estimates; standard errors for WR/LC are
		typically $1.0$--$1.2$ percentage points in our runs.
	}
	\label{tab:zstar_ablation}
	\begin{tabular}{lcccc}
		\toprule
		Model & $z^*$ / Method & Target signal $\sigma(-z^*)$ & WR (\%) & LC (\%) \\
		\midrule
		\multirow{9}{*}{Llama3-8B}
		& 0.20 & 0.4502 & 3.8 & 4.3 \\
		& 0.30 & 0.4256 & 12.8 & 14.7 \\
		& 0.35 & 0.4134 & 20.1 & 20.9 \\
		& \textbf{0.40} & \textbf{0.4013} & \textbf{21.1} & \textbf{22.7} \\
		& 0.45 & 0.3894 & 19.7 & 21.1 \\
		& 0.50 & 0.3775 & 18.6 & 20.7 \\
		& 0.60 & 0.3543 & 18.7 & 20.2 \\
		& 0.70 & 0.3318 & 17.5 & 19.5 \\
		& DPO & 0.3429 & 17.7 & 19.2 \\
		\midrule
		\multirow{9}{*}{Qwen3-8B}
		& 0.20 & 0.4502 & 12.9 & 18.6 \\
		& 0.30 & 0.4256 & 15.2 & 21.1 \\
		& \textbf{0.35} & \textbf{0.4134} & \textbf{16.1} & \textbf{22.1} \\
		& 0.40 & 0.4013 & 13.5 & 18.9 \\
		& 0.45 & 0.3894 & 12.9 & 18.6 \\
		& 0.50 & 0.3775 & 12.3 & 17.8 \\
		& 0.60 & 0.3543 & 10.9 & 18.2 \\
		& 0.70 & 0.3318 & 8.8 & 13.7 \\
		& DPO & 0.3787 & 13.1 & 17.2 \\
		\bottomrule
	\end{tabular}
\end{table}

\subsection{Ablation of Control Parameters}
We further study the parameters of LSC-DPO on Anthropic-HH using Llama3-8B and Qwen3-8B. Figures~\ref{fig:ablation_all_llama} and \ref{fig:ablation_all_qwen} vary the target scaled margin $z^*$, the control rate $\eta$, and the start threshold $\tau_{\mathrm{start}}$.

The target scaled margin $z^*$ specifies the desired learning-signal regime
through $\sigma(-z^*)$. Moderate values of $z^*$ generally achieve stronger win rates, whereas overly large values reduce performance, consistent with the view that a small target signal moves the loss toward a lower-sensitivity regime. The control rate $\eta$ determines the responsiveness of the feedback update. Moderate values perform well, while overly large rates can degrade performance, consistent with the stability dependence on $\eta z^*$ in Proposition~\ref{prop:local-stability}. The start threshold
$\tau_{\mathrm{start}}$ determines when the dynamic update is activated.
Compared with $z^*$ and $\eta$, performance is less sensitive to $\tau_{\mathrm{start}}$ within the tested range, although moderate activation
thresholds tend to perform best.

Overall, these ablations suggest that LSC-DPO tolerates moderate variations in its control parameters, while benefiting from target margins and control rates that maintain a responsive learning-signal regime.
\begin{figure}[H]
	\centering
	\begin{subfigure}{0.32\textwidth}
		\centering
		\includegraphics[width=\linewidth]{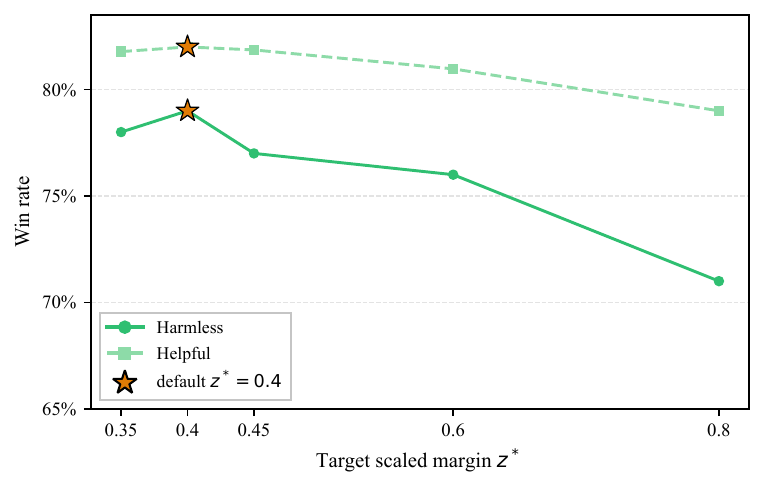}
		\caption{Target scaled margin $z^*$}
		\label{fig:ablation_zstar}
	\end{subfigure}
	\hfill
	\begin{subfigure}{0.32\textwidth}
		\centering
		\includegraphics[width=\linewidth]{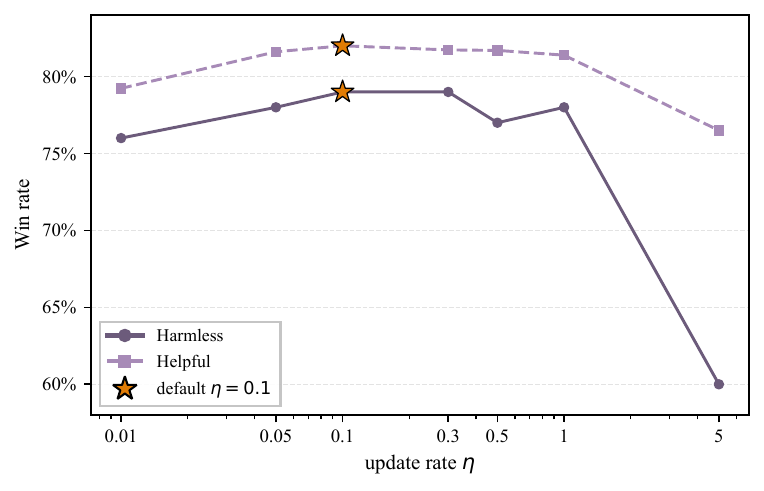}
		\caption{Control rate $\eta$}
		\label{fig:ablation_eta}
	\end{subfigure}
	\hfill
	\begin{subfigure}{0.32\textwidth}
		\centering
		\includegraphics[width=\linewidth]{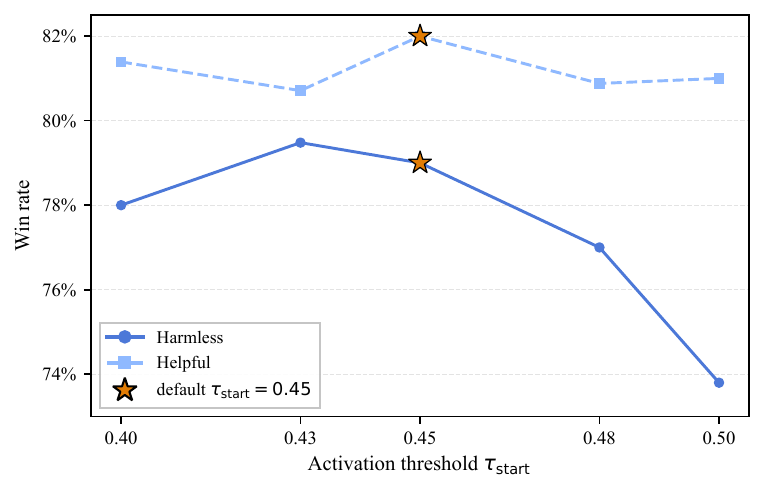}
		\caption{Start threshold $\tau_{\mathrm{start}}$}
		\label{fig:ablation_tau_target}
	\end{subfigure}
	
	\caption{
		Ablation analysis of key control parameters with Llama3-8B on Anthropic-HH subsets: (a) target scaled margin $z^*$, (b) control rate $\eta$, and (c) start threshold $\tau_{\mathrm{start}}$. The star marks the default setting used in the main experiments.
	}
	\label{fig:ablation_all_llama}
\end{figure}
\begin{figure}[H]
	\centering
	\begin{subfigure}{0.32\textwidth}
		\centering
		\includegraphics[width=\linewidth]{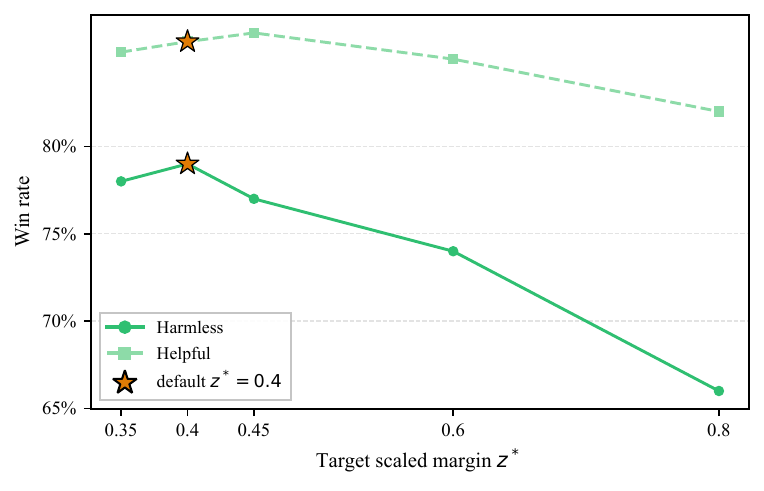}
		\caption{Target scaled margin $z^*$}
		\label{fig:ablation_zstar_qwen}
	\end{subfigure}
	\hfill
	\begin{subfigure}{0.32\textwidth}
		\centering
		\includegraphics[width=\linewidth]{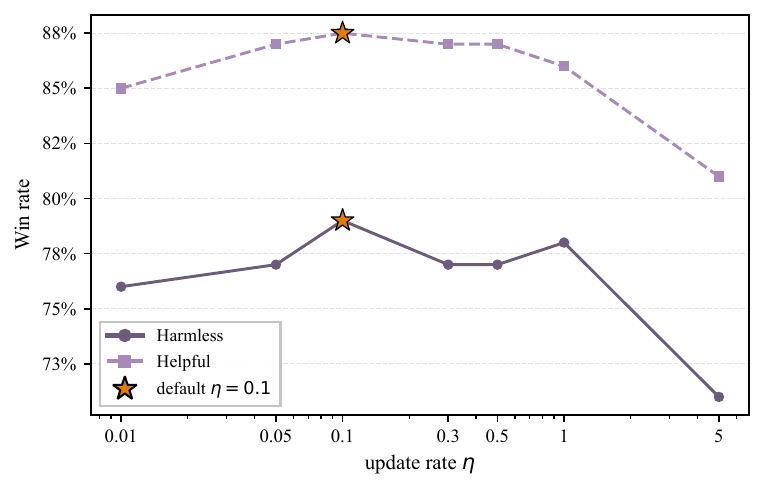}
		\caption{Control rate $\eta$}
		\label{fig:ablation_eta_qwen}
	\end{subfigure}
	\hfill
	\begin{subfigure}{0.32\textwidth}
		\centering
		\includegraphics[width=\linewidth]{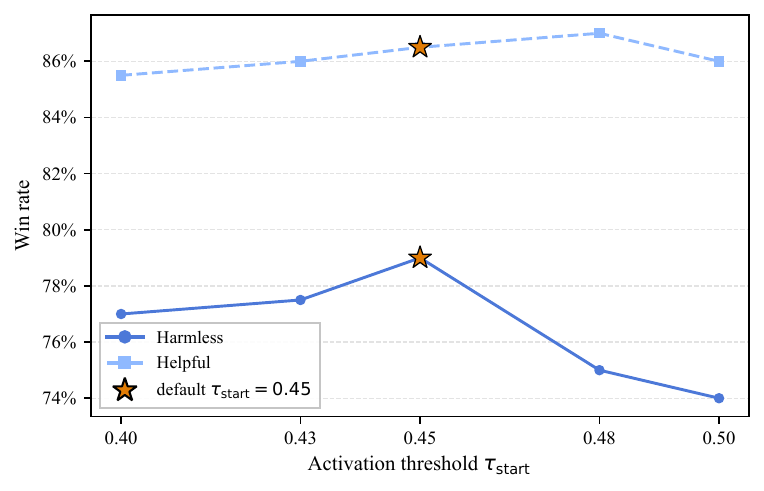}
		\caption{Start threshold $\tau_{\mathrm{start}}$}
		\label{fig:ablation_tau_target_qwen}
	\end{subfigure}
	
	\caption{
		Ablation analysis of key control parameters with Qwen3-8B on Anthropic-HH subsets: (a) target scaled margin $z^*$, (b) control rate $\eta$, and (c) start threshold $\tau_{\mathrm{start}}$. The star marks the default setting used in the main experiments.
	}
	\label{fig:ablation_all_qwen}
\end{figure}

\section{Downstream Evaluation}
\label{downstream}
Following the downstream evaluation tasks used in SimPO~\citep{meng2024simpo} and $\varepsilon$-DPO~\citep{lee2025kl}, we further evaluate whether preference optimization preserves general task capabilities beyond open-ended chatbot benchmarks. We consider MMLU~\citep{hendrycks2020measuring}, ARC-Challenge~\citep{clark2018think}, HellaSwag~\citep{zellers2019hellaswag}, TruthfulQA~\citep{lin2022truthfulqa}, WinoGrande~\citep{sakaguchi2021winogrande}, and GSM8K~\citep{cobbe2021training}, covering knowledge, reading comprehension, commonsense reasoning, truthfulness, and mathematical reasoning. Tables~\ref{tab:downstream_eval} and~\ref{tab:downstream_eval_qwen} report the results for Llama3-8B and Qwen3-8B, respectively.

Overall, LSC-DPO maintains downstream task performance at a level comparable to strong preference-optimization baselines and does not exhibit a consistent degradation relative to SFT. The effect of preference optimization varies across tasks and model families, consistent with prior observations that downstream task
changes can depend on the pretrained model, preference data, and optimization objective. Notably, on GSM8K, LSC-DPO does not show a score drop relative to the SFT initialization on either model family, suggesting that the alignment improvements do not come from a substantial loss of mathematical reasoning capability. As an additional reproducibility check, we also evaluate the released SimPO checkpoints on the same set of downstream tasks under the same evaluation protocol and obtain highly similar results. This consistency strengthens the reliability of our capability-preservation conclusion. We therefore interpret these downstream results mainly as a capability-preservation check rather than as the primary evidence for alignment improvement.

\begin{table}[H]
	\centering
	\scriptsize
	\setlength{\tabcolsep}{3pt}
	\caption{
		Downstream capability evaluation on Llama3-8B. Scores are reported as percentages. Avg. denotes the mean over the displayed task scores.
	}
	\label{tab:downstream_eval}
	\resizebox{\textwidth}{!}{
		\begin{tabular}{lccccccc}
			\toprule
			\textbf{Method} &
			\textbf{MMLU} &
			\textbf{ARC-C} &
			\textbf{HellaSwag} &
			\textbf{TruthfulQA} &
			\textbf{WinoGrande} &
			\textbf{GSM8K} &
			\textbf{Avg.} \\
			\midrule
			SFT & 62.54 & 60.24 & 81.30 & 45.50 & 74.43 & 49.20 & 62.20 \\
			DPO & 63.15 & 64.51 & 83.14 & 53.12 & 75.14 & 54.97 & 65.67 \\
			$\beta$-DPO & 62.92 & 63.99 & 83.04 & 51.37 & 75.14 & 53.45 & 64.99 \\
			$\varepsilon$-DPO & 62.98 & 64.33 & 83.29 & 52.94 & 75.14 & 53.15 & 65.30 \\
			IPO & 62.96 & 62.63 & 81.68 & 49.02 & 73.88 & 54.36 & 64.09 \\
			CPO & 62.73 & 60.32 & 81.44 & 52.97 & 75.06 & 48.22 & 63.46 \\
			KTO & 62.88 & 64.51 & 83.39 & 53.93 & 75.85 & 53.37 & 65.65 \\
			ORPO & 62.80 & 55.29 & 80.70 & 47.72 & 74.03 & 44.35 & 60.82 \\
			SLiC-HF & 62.69 & 57.85 & 80.57 & 50.82 & 74.35 & 46.47 & 62.13 \\
			R-DPO & 62.41 & 61.95 & 80.91 & 43.74 & 74.35 & 39.58 & 60.49 \\
			SimPO & 62.85 & 65.78 & 83.12 & 58.48 & 74.25 & 52.54 & 66.17 \\
			LSC-DPO & 63.02 & 65.25 & 83.34 & 52.57 & 74.82 & 53.98 & 65.50 \\
			\bottomrule
		\end{tabular}
	}
\end{table}

\begin{table}[H]
	\centering
	\scriptsize
	\setlength{\tabcolsep}{3pt}
	\caption{
		Downstream capability evaluation on Qwen3-8B. Scores are reported as percentages. Avg. denotes the mean over the displayed task scores.
	}
	\label{tab:downstream_eval_qwen}
	\resizebox{\textwidth}{!}{
		\begin{tabular}{lccccccc}
			\toprule
			\textbf{Method} &
			\textbf{MMLU} &
			\textbf{ARC-C} &
			\textbf{HellaSwag} &
			\textbf{TruthfulQA} &
			\textbf{WinoGrande} &
			\textbf{GSM8K} &
			\textbf{Avg.} \\
			\midrule
			SFT & 75.99 & 67.49 & 80.17 & 53.87 & 75.85 & 82.18 & 72.59 \\
			DPO & 76.54 & 68.69 & 81.30 & 57.39 & 76.01 & 84.46 & 74.06 \\
			$\beta$-DPO & 76.48 & 68.77 & 81.13 & 56.40 & 76.72 & 83.40 & 73.82 \\
			$\varepsilon$-DPO & 76.63 & 68.34 & 81.44 & 57.57 & 76.01 & 84.53 & 74.09 \\
			IPO & 76.48 & 67.83 & 80.92 & 56.30 & 76.09 & 84.15 & 73.63 \\
			CPO & 75.63 & 68.09 & 80.03 & 56.72 & 75.61 & 83.55 & 73.27 \\
			KTO & 76.41 & 68.77 & 81.12 & 57.37 & 76.09 & 84.99 & 74.12 \\
			ORPO & 75.84 & 66.38 & 79.66 & 53.83 & 75.22 & 67.48 & 69.73 \\
			SLiC-HF & 75.01 & 67.06 & 79.33 & 55.03 & 74.51 & 79.98 & 71.82 \\
			R-DPO & 75.92 & 68.00 & 80.60 & 54.07 & 75.93 & 84.23 & 73.13 \\
			SimPO & 76.37 & 68.69 & 81.24 & 56.73 & 76.80 & 84.84 & 74.11 \\
			LSC-DPO & 76.50 & 68.52 & 81.38 & 57.50 & 76.01 & 84.53 & 74.07 \\
			\bottomrule
		\end{tabular}
	}
\end{table}

\end{document}